\documentclass[english]{article}

\usepackage{algorithmic}
\usepackage{algorithm}
\usepackage{amsfonts}
\usepackage{amsmath}
\usepackage{amssymb}
\usepackage{amsthm}
\usepackage{array}
\usepackage{arydshln}
\usepackage{bm}
\usepackage{cite}
\usepackage{color}
\usepackage{comment}
\usepackage{dsfont}
\usepackage{enumitem}
\usepackage{float}
\usepackage[T1]{fontenc}
\usepackage{geometry}
\usepackage{graphicx}
\usepackage{hyperref}\hypersetup{colorlinks,linkcolor=red,anchorcolor=blue,citecolor=blue}
\usepackage[latin9]{inputenc}
\usepackage{letltxmacro}
\usepackage{mathrsfs}
\usepackage{mathtools}
\usepackage{multirow}
\usepackage{nicefrac}   
\usepackage{verbatim}

\newcommand{\ba}{\bm{a}}

\newcommand{\br}{\bm{r}}
\newcommand{\bs}{\bm{s}}

\newcommand{\bv}{\bm{v}}
\newcommand{\bw}{\bm{w}}

\newcommand{\by}{\bm{y}}

\newcommand{\bA}{\bm{A}}
\newcommand{\bB}{\bm{B}}

\newcommand{\bF}{\bm{F}}

\newcommand{\bL}{\bm{L}}
\newcommand{\bM}{\bm{M}}

\newcommand{\bQ}{\bm{Q}}
\newcommand{\bR}{\bm{R}}
\newcommand{\bS}{\bm{S}}

\newcommand{\bU}{\bm{U}}
\newcommand{\bV}{\bm{V}}

\newcommand{\bX}{\bm{X}}

\newcommand{\bDelta}{\bm{\Delta}}

\newcommand{\bSigma}{\bm{\Sigma}}

\newcommand{\cA}{\mathcal{A}}

\newcommand{\cI}{\mathcal{I}}

\newcommand{\cL}{\mathcal{L}}

\newcommand{\cN}{\mathcal{N}}

\newcommand{\cP}{\mathcal{P}}

\newcommand{\cS}{\mathcal{S}}

\newcommand{\RR}{\mathbb{R}}

\newcommand{\mfk}{\mathfrak} 

\newcommand{\zero}{\bm{0}}

\newcommand{\argmin}{\mathop{\mathrm{argmin}}}
\newcommand{\minimize}{\mathop{\mathrm{minimize}}}

\DeclareMathOperator{\dist}{\mathrm{dist}}
\DeclareMathOperator{\fro}{\mathsf{F}}
\DeclareMathOperator{\GL}{\mathrm{GL}}

\DeclareMathOperator{\op}{\mathsf{op}}
\DeclareMathOperator{\rank}{\mathrm{rank}}

\DeclareMathOperator{\tr}{\mathrm{tr}}

\allowdisplaybreaks
\makeatletter
\providecommand{\tabularnewline}{\\}

\let\tilde\widetilde

\theoremstyle{plain}\newtheorem{lemma}{\textbf{Lemma}} 
\newtheorem{proposition}{\textbf{Proposition}}
\newtheorem{theorem}{\textbf{Theorem}}
\newtheorem{corollary}{\textbf{Corollary}}

\theoremstyle{definition}\newtheorem{definition}{\textbf{Definition}}
 
\theoremstyle{remark}\newtheorem{remark}{\textbf{Remark}}

\definecolor{tian}{RGB}{0,150,0}
\definecolor{cm}{RGB}{250,0,200}
\definecolor{yc}{RGB}{250,25,20}

\begin{document}

\title{Low-Rank Matrix Recovery with Scaled Subgradient Methods: Fast and Robust Convergence Without the Condition Number  }

\author
{
	Tian Tong\thanks{Department of Electrical and Computer Engineering, Carnegie Mellon University, Pittsburgh, PA 15213, USA; Emails:
		\texttt{\{ttong1,yuejiec\}@andrew.cmu.edu}.} \\
		Carnegie Mellon University \\
		\and
	Cong Ma\thanks{Department of Statistics and Department of Electrical Engineering and Computer Sciences, UC Berkeley, Berkeley, CA 94720, USA; Email:
		\texttt{congm@berkeley.edu}.} \\
		UC Berkeley \\
		\and
	Yuejie Chi\footnotemark[1] \\
	Carnegie Mellon University
}
 \date{October 2020; Revised April 2021}

\maketitle

\begin{abstract}

Many problems in data science can be treated as estimating a low-rank matrix from highly incomplete, sometimes even corrupted, observations. One popular approach is to resort to matrix factorization, where the low-rank matrix factors are optimized via first-order methods over a smooth loss function, such as the residual sum of squares. While tremendous progress has been made in recent years, the natural smooth formulation suffers from two sources of ill-conditioning, where the iteration complexity of gradient descent scales poorly both with the dimension as well as the condition number of the low-rank matrix. Moreover, the smooth formulation is not robust to corruptions. In this paper, we propose {\em scaled} subgradient methods to minimize a family of nonsmooth and nonconvex formulations---in particular, the residual sum of absolute errors---which is guaranteed to converge at a fast rate that is almost dimension-free and independent of the condition number, even in the presence of corruptions.
We illustrate the effectiveness of our approach when the observation operator satisfies certain mixed-norm restricted isometry properties, and derive state-of-the-art performance guarantees for a variety of problems such as robust low-rank matrix sensing and quadratic sampling.

\medskip


\end{abstract}

\section{Introduction}

Many problems in data science can be treated as estimating a low-rank matrix $\bX_\star\in\RR^{n_1\times n_2}$ from highly incomplete, sometimes even corrupted, observations $\by=\{y_i\}_{i=1}^m$ given by
\begin{align} 
y_i \approx \cA_i(\bX_\star), \qquad 1\le i\le m.
\end{align} 
Here, $\cA(\cdot) =\{\cA_i(\cdot)\}_{i=1}^m : \RR^{n_1 \times n_2} \mapsto \RR^{m}$ is the observation operator that models the measurement process. Instead of operating in the full matrix space, i.e.~$\RR^{n_1 \times n_2}$, a memory-efficient approach is to resort to low-rank matrix factorization, by writing $\bX_{\star} = \bL_{\star}\bR_{\star}^{\top}$, if the rank $r$ of $\bX_{\star}$ is known \emph{a priori}, where $\bL_{\star}\in \RR^{n_1\times r}$ and $\bR_{\star}\in \RR^{n_2\times r}$ are of a size that is proportional to the degrees of freedom of the low-rank matrix. Furthermore, the low-rank factors can be found by optimizing a smooth loss function, such as the residual sum of squares
\begin{align}
\minimize_{\bL\in\RR^{n_1 \times r},\bR\in\RR^{n_2 \times r}} \quad \sum_{i=1}^{m} \left(\cA_i(\bL\bR^{\top}) - y_i\right)^2,\label{eq:rss}
\end{align}
using first-order methods (e.g.~gradient descent). While tremendous progress has been made in recent years \cite{chi2019nonconvex}, applying vanilla gradient descent to the above smooth formulation suffers from two sources of ill-conditioning that preclude a desirable computational efficiency from classical optimization principles: 
\begin{itemize}
\item Due to the heavy-tailed nature of certain measurement operators, such as those encountered in phase retrieval \cite{candes2015phase} and quadratic sampling \cite{sanghavi2017local}, the least-squares formulation \eqref{eq:rss} may suffer from a large smoothness parameter (and hence a large condition number of the loss function) that scales at least linearly with respect to the ambient dimension, leading to a conservative choice of stepsizes and a high iteration complexity when the problem dimension is large. 
\item Due to the composite nature of the formulation \eqref{eq:rss}, the iteration complexity of vanilla gradient descent is further exacerbated by the condition number of the underlying low-rank matrix $\bX_{\star}$, which could be large in many applications of interest.
\end{itemize}
While there have been encouraging activities \cite{charisopoulos2019low,ma2017implicit,li2018nonconvex,tong2020accelerating} that try to alleviate these issues regarding ill-conditioning, none of the existing first-order approaches are able to simultaneously remove both sources of ill-conditioning and achieve fast convergence. Therefore, the goal of the current paper is to develop first-order methods that are guaranteed to converge at a fast rate that is almost dimension-free and independent of the condition number, even in the presence of corruptions.

\subsection{Main contributions}

In this paper, we propose to minimize the following {\em nonsmooth and nonconvex} loss function known as the {\em least absolute deviations}, which measures the residual sum of absolute errors 
\begin{align}
\minimize_{\bL\in \RR^{n_1 \times r},\bR \in \RR^{n_2 \times r}} \quad f(\bL\bR^{\top}) \coloneqq \sum_{i=1}^{m} \left|\cA_i(\bL\bR^{\top})-y_i\right|,
\end{align}
via a {\em scaled subgradient method}:
\begin{align}
\begin{split} 
\bL_{t+1} & \coloneqq\bL_{t}-\eta_{t}\bS_{t}\bR_{t}(\bR_{t}^{\top}\bR_{t})^{-1}, \\
\bR_{t+1} & \coloneqq\bR_{t}-\eta_{t}\bS_{t}^{\top}\bL_{t}(\bL_{t}^{\top}\bL_{t})^{-1}.
\end{split}\label{eq:ssm}
\end{align}
Here, $\bS_{t}\in\partial f(\bL_{t}\bR_{t}^{\top})$ is a subgradient of $f(\bX) \coloneqq \sum_{i=1}^{m} \left|\cA_i(\bX)-y_i\right|$ at $\bL_{t}\bR_{t}^{\top}$, and $\eta_t >0$ is a sequence of carefully-chosen stepsizes. Compared with vanilla subgradient methods, our new method \eqref{eq:ssm} scales or preconditions the search directions $\bS_{t}\bR_{t}$ and $\bS_{t}^{\top}\bL_{t}$ by $(\bR_{t}^{\top}\bR_{t})^{-1}$ and $(\bL_{t}^{\top}\bL_{t})^{-1}$, respectively.\footnote{Under appropriate conditions, the inverse matrices always exist; in practice, one can use the pseudo-inverse matrices to avoid numerical instabilities.} As explained in \cite{tong2020accelerating} where a similar preconditioning trick was employed for smooth formulations, the scaled subgradient enables better search directions and therefore larger stepsizes. Our main results can be summarized as follows:
\begin{itemize}
\item Under general geometric assumptions on $f(\cdot)$ such as restricted rank-$r$ Lipschitz continuity and sharpness conditions, we demonstrate that the convergence rate of scaled subgradient methods using both Polyak's and geometrically decaying stepsizes is {\em independent} of the condition number of $\bX_{\star}$. 
\item Instantiating our theory under the mixed-norm restricted isometry property (RIP) of the measurement operator, we demonstrate state-of-the-art computational guarantees for low-rank matrix sensing and quadratic sampling even when the observations are noisy and corrupted by outliers. This leads to improvements over the computational complexity of scaled gradient methods in \cite{tong2020accelerating} for heavy-tailed measurement ensembles, as well as of vanilla subgradient methods in \cite{charisopoulos2019low}. Table~\ref{tab:Performance-guarantees} provides a detailed comparison of the local iteration complexities of the proposed scaled subgradient method in comparison with these prior algorithms, highlighting its robustness to heavy-tailed observations, outliers, as well as a large condition number of the true matrix $\bX_{\star}$.
\end{itemize}
Our work leverages exciting advances in nonsmooth optimization \cite{charisopoulos2019low} and scaled first-order methods \cite{tong2020accelerating} for low-rank matrix recovery. Our arguments are concise, which avoid the need of sophisticated trajectory-dependent analysis as have been used in \cite{ma2017implicit,li2018nonconvex} to achieve rapid and robust convergence guarantees.

\begin{table}[t]
\centering %
\begin{tabular}{c|c|c||c|c}
\hline 
 & \multicolumn{2}{c||}{matrix sensing}    & \multicolumn{2}{c}{quadratic sampling}\tabularnewline 
\hline 
Algorithms & without corruptions & with corruptions & without corruptions & with corruptions \tabularnewline
\hline \hline
\texttt{GD}  & \multirow{2}{*}{$\kappa\log\frac{1}{\epsilon}$} & \multirow{2}{*}{N/A} &  \multirow{2}{*}{$r^2\kappa^2 \log\frac{1}{\epsilon}$} & \multirow{2}{*}{N/A}\tabularnewline
\cite{tu2015low,li2018nonconvex}  &  &  &  & \tabularnewline
\hline 
\texttt{ScaledGD}  & \multirow{2}{*}{$\log\frac{1}{\epsilon}$} & \multirow{2}{*}{N/A} &  \multirow{2}{*}{$\mathrm{poly}(n) \log\frac{1}{\epsilon}$} & \multirow{2}{*}{N/A}\tabularnewline
\cite{tong2020accelerating}  &  &  &  & \tabularnewline
\hline 
\texttt{SM}  & \multirow{2}{*}{$\kappa\log\frac{1}{\epsilon}$} & \multirow{2}{*}{$\frac{\kappa}{(1-2p_s)^2}\log\frac{1}{\epsilon}$}   & \multirow{2}{*}{$r \kappa \log\frac{1}{\epsilon}$} & \multirow{2}{*}{$\frac{r\kappa}{(1-2p_s)^2}\log\frac{1}{\epsilon}$}\tabularnewline
\cite{charisopoulos2019low,li2020nonconvex}  &  &  &  & \tabularnewline
\hline 
\texttt{ScaledSM}  & \multirow{2}{*}{$\log\frac{1}{\epsilon}$} & \multirow{2}{*}{$\frac{1}{(1-2p_s)^2}\log\frac{1}{\epsilon}$}   & \multirow{2}{*}{$r \log\frac{1}{\epsilon}$} & \multirow{2}{*}{$\frac{r}{(1-2p_s)^2}\log\frac{1}{\epsilon}$}\tabularnewline
(this paper)  &  &  &  & \tabularnewline
\hline
\end{tabular}\vspace{0.04in}
\caption{Local iteration complexities of the proposed scaled subgradient method (\texttt{ScaledSM}) in comparison with prior algorithms for matrix sensing and quadratic sampling. \texttt{ScaledSM} outperforms the vanilla subgradient method (\texttt{SM}) by a factor of $\kappa$ in both problems, while outperforms scaled gradient descent (\texttt{ScaledGD}), and \texttt{GD} with additional robustness guarantees. Here, $n=\max\{n_1,n_2\}$, $r$ is the rank, $\kappa$ is the condition number of $\bX_{\star}$, and $0\le p_s<1/2$ is the fraction of outliers. We say that the output $\bX$ of an algorithm reaches $\epsilon$-accuracy, if it satisfies $\|\bX-\bX_{\star}\|_{\fro}\le\epsilon\sigma_{r}(\bX_{\star})$, where $\sigma_{r}(\bX_{\star})$ denotes the $r$-th largest singular value of $\bX_{\star}$.
\label{tab:Performance-guarantees}}
\end{table}

\subsection{Related work}

Low-rank matrix recovery has been a target of intense interest in the last decade; we invite the readers to \cite{davenport2016overview,chen2018harnessing,chi2019nonconvex} for recent overviews, and limit our discussions to the most relevant literature in the sequel.

\paragraph{Nonsmooth formulations for low-rank matrix recovery.} Nonsmooth objective functions, such as the least absolute deviations, have been adopted earlier in both convex and nonconvex formulations of low-rank matrix recovery, including phase retrieval \cite{hand2017phaselift,davis2017nonsmooth,qu2017convolutional,zhang2017reshaped,duchi2019solving}, blind deconvolution \cite{diaz2019nonsmooth}, quadratic sampling \cite{li2017low,chi2016kaczmarz,charisopoulos2019low,bahmani2020low}, low-rank matrix sensing \cite{charisopoulos2019low,li2013compressed,wright2013compressive,li2020nonconvex}, robust synchronization \cite{wang2013exact}, to name a few. 
Our work is most closely related to and generalizes the vanilla subgradient method in \cite{charisopoulos2019low}, by establishing novel performance guarantees of \emph{scaled} subgradient methods for robust low-rank matrix recovery.   

\paragraph{Scaled first-order methods for low-rank matrix recovery.} Variants of the scaled gradient methods are proposed in \cite{mishra2012riemannian,tanner2016low,tong2020accelerating} for minimizing the least-squares formulation \eqref{eq:rss}, where strong statistical and computational complexities are first established in \cite{tong2020accelerating}. To the best of our knowledge, the current paper is the first work that provides rigorous statistical and computational guarantees for scaled subgradient methods for addressing nonsmooth formulations. When it comes to problems with heavy-tail observations such as quadratic sampling, while it is possible to establish faster convergence rates of vanilla gradient descent over the smooth least-squares loss function through a tailored analysis \cite{ma2017implicit,li2018nonconvex} via leave-one-out arguments, it is unclear if similar analyses are viable for scaled gradient methods (\texttt{ScaledGD}) in \cite{tong2020accelerating}. Unfortunately, a direct application of the performance guarantee of \texttt{ScaledGD} on minimizing the smooth least-squares loss function leads to a much slower rate in terms of the problem dimension (see Table~\ref{tab:Performance-guarantees}) for quadratic sampling. In contrast, our analysis for scaled subgradient methods yields strong guarantees in a more straightforward manner since the nonsmooth loss function has much better geometric properties \cite{charisopoulos2019low}.

\paragraph{Robust low-rank matrix recovery via nonconvex optimization.} A pleasant side benefit of nonsmooth formulations is the added robustness to adversarial outliers under a simple algorithm design -- the low-rank factors are updated essentially in the same manner regardless of the presence of outliers. In comparison, other nonconvex methods based on smooth formulations often need to introduce some special treatments to mitigate outliers before updating the low-rank factors, e.g.~truncation or thresholding \cite{zhang2016provable,li2020non,li2020nonconvex}, which can be cumbersome to tune properly.
 
\paragraph{Condition number independent rate of convergence.} It is well-known that first-order methods such as gradient descent exhibit poor scaling with respect to the condition number of the low-rank matrix. Possible remedies include alternating least-squares in the factored space \cite{jain2013low,hardt2014fast}, or spectral methods over the matrix space \cite{jain2010guaranteed}. However, these approaches either require the inversion of a large matrix or a higher memory footprint, compared with the scaled first-order methods adopted herein.

\subsection{Paper organization and notation}

The rest of this paper is organized as follows. Section~\ref{sec:problem_formulation} describes the proposed scaled subgradient method and its connections to existing methods. 
Section~\ref{sec:main-results} provides the theoretical guarantees for the scaled subgradient method in terms of both statistical and computational complexities, which are then instantiated to robust low-rank matrix sensing and quadratic sampling. Section~\ref{sec:numerical} illustrates the superior empirical performance of the proposed method. Finally, we conclude in Section~\ref{sec:discussion}. The proofs are deferred to the appendix. 

\paragraph{Notation.} Throughout the paper, we use boldfaced symbols for vectors and matrices. For a vector $\bv$, we use $\|\bv\|_{p}$ to denote its $\ell_p$ norm. For any matrix $\bA$, we use $\sigma_{i}(\bA)$ to denote its $i$-th largest singular value, and let $\bA_{i,\cdot}$ and $\bA_{\cdot,j}$ denote its $i$-th row and $j$-th column, respectively. In addition, $\|\bA\|_{\op}$ and $\|\bA\|_{\fro}$ denote the spectral norm and the Frobenius norm of a matrix $\bA$, respectively. For matrices $\bA,\bB$ of the same size, we use $\langle\bA,\bB\rangle=\sum_{i,j}\bA_{i,j}\bB_{i,j}=\tr(\bA^{\top}\bB)$ to denote their inner product, where $\tr(\cdot)$ denotes the trace. The set of invertible matrices in $\RR^{r\times r}$ is denoted by $\GL(r)$.

\section{Problem Formulation and Algorithms} \label{sec:problem_formulation}

In this section, we formulate the low-rank matrix recovery problem, followed by a detailed description of the proposed scaled subgradient method.

\subsection{Problem formulation}

Let $\bX_{\star}\in\RR^{n_{1}\times n_{2}}$ be the ground truth rank-$r$ matrix, whose compact singular value decomposition (SVD) is given by
\begin{align}
\bX_{\star} = \bU_{\star}\bSigma_{\star}\bV_{\star}^{\top},
\end{align} 
where $\bU_{\star}\in\RR^{n_{1}\times r}$ and $\bV_{\star}\in\RR^{n_{2}\times r}$ are composed of $r$ left and right singular vectors, respectively, and $\bSigma_{\star}\in\RR^{r\times r}$ is a diagonal matrix consisting of $r$ singular values of $\bX_{\star}$ organized in a non-increasing order, i.e.~$\sigma_{1}(\bX_{\star})\ge\dots\ge\sigma_{r}(\bX_{\star})>0$. The condition number of $\bX_{\star}$ is thus defined as
\begin{align}
\kappa\coloneqq\sigma_{1}(\bX_{\star})/\sigma_{r}(\bX_{\star}).\label{eq:kappa}
\end{align} 
Without loss of generality, we define the ground truth low-rank factors as 
\begin{align}
\bL_{\star}\coloneqq\bU_{\star}\bSigma_{\star}^{1/2},\quad \text{and} \quad \bR_{\star}\coloneqq\bV_{\star}\bSigma_{\star}^{1/2},
\end{align}
so that $\bX_{\star}=\bL_{\star}\bR_{\star}^{\top}$. Moreover, we denote the ground truth stacked factor matrix as 
\begin{align}
\bF_{\star}\coloneqq[\bL_{\star}^{\top},\bR_{\star}^{\top}]^{\top}\in\RR^{(n_{1}+n_{2})\times r}.
\end{align}

Assume that we have access to a number of observations $\by =\{y_i\}_{i=1}^m$ of $\bX_{\star}$, given as
\begin{align}
y_i = \cA_i(\bX_{\star}) + w_i + s_i, \qquad 1\le i\le m,\label{eq:entrywise_observation}
\end{align}
or equivalently,
\begin{align}
\by = \cA(\bX_{\star}) + \bw + \bs,\label{eq:observation_model}
\end{align}
where $\cA(\bX_{\star})=\{\cA_i(\bX_{\star})\}_{i=1}^m$ is the measurement ensemble, $\bw=\{w_i\}_{i=1}^m$ denotes the bounded noise, and $\bs=\{s_i\}_{i=1}^m$ models arbitrary corruptions. The goal of low-rank matrix recovery is to reconstruct $\bX_{\star}$ from the noisy and corrupted observations $\by$ in a statistically and computationally efficient manner.

\subsection{Scaled subgradient method}

Consider the following nonsmooth and nonconvex optimization problem over the factors
\begin{align}
\minimize_{\bL\in\RR^{n_{1}\times r}, \bR\in\RR^{n_{2}\times r}}\quad\cL(\bL,\bR)\coloneqq f(\bL\bR^{\top}),\label{eq:opt}
\end{align}
where $f(\cdot)$ is a nonsmooth surrogate of the observation residuals. Of particular interest is the residual sum of absolute errors, defined as
\begin{align}
f(\bX) \coloneqq \|\cA(\bX)-\by\|_{1}.\label{eq:loss}
\end{align}
Correspondingly, the minimizer is called the least absolute deviations (LAD) solution.

Let us denote the stacked factor matrix in the $t$-th iterate as $\bF_{t}\coloneqq[\bL_{t}^{\top},\bR_{t}^{\top}]^{\top}$. Given an initialization $\bF_{0}=[\bL_{0}^{\top},\bR_{0}^{\top}]^{\top}$, the proposed scaled subgradient method (\texttt{ScaledSM}) proceeds as 
\begin{align}
\begin{split}
\bL_{t+1} & \coloneqq\bL_{t}-\eta_{t}\bS_{t}\bR_{t}(\bR_{t}^{\top}\bR_{t})^{-1}, \\
\bR_{t+1} & \coloneqq\bR_{t}-\eta_{t}\bS_{t}^{\top}\bL_{t}(\bL_{t}^{\top}\bL_{t})^{-1},
\end{split}\label{eq:scaledSM}
\end{align}
where $\bS_{t}\in\partial f(\bL_{t}\bR_{t}^{\top})$ is a subgradient of $f(\cdot)$ at $\bL_{t}\bR_{t}^{\top}$ (and hence $\bS_t\bR_t \in \partial_{\bL}\cL(\bL_t,\bR_t)$ and $\bS_t^{\top}\bL_t \in \partial_{\bR}\cL(\bL_t,\bR_t)$), and $\eta_{t}>0$ is some properly selected stepsize, which we discuss next.

\paragraph{Stepsize schedules.} We consider the following two choices of stepsize schedules:
\begin{itemize}
\item If we know the optimal value $f(\bX_\star)$, we can invoke the following Polyak's stepsize, given by
\begin{align}
\eta_{t}^{\mathsf{P}}\coloneqq\frac{f(\bL_{t}\bR_{t}^{\top})-f(\bX_{\star})}{\|\bS_{t}\bR_{t}(\bR_{t}^{\top}\bR_{t})^{-1/2}\|_{\fro}^{2}+\|\bS_{t}^{\top}\bL_{t}(\bL_{t}^{\top}\bL_{t})^{-1/2}\|_{\fro}^{2}},\label{eq:Polyak_stepsize}
\end{align}
where the denominator is the squared norm of the subgradient under a scaled metric concerted with the preconditioners. This schedule is implementable, for example, when the observations are noise-free, leading to $f(\bX_\star)=0$. However, when the observations are noisy and corrupted, it is not viable to know $f(\bX_\star)$ beforehand.

\item In general, we can apply the geometrically decaying stepsize originally introduced in \cite{goffin1977convergence}, given by 
\begin{align}
\eta_{t}^{\mathsf{G}}\coloneqq\frac{\lambda q^{t}}{\sqrt{\|\bS_{t}\bR_{t}(\bR_{t}^{\top}\bR_{t})^{-1/2}\|_{\fro}^{2}+\|\bS_{t}^{\top}\bL_{t}(\bL_{t}^{\top}\bL_{t})^{-1/2}\|_{\fro}^{2}}},\label{eq:geometric_stepsize} 
\end{align}
where the denominator is similarly scaled as \eqref{eq:Polyak_stepsize}, and $\lambda>0$ and $q\in(0,1)$ are some parameters to be specified. This choice is broadly applicable when dealing with noisy and corrupted observations.
\end{itemize}

Compared with the vanilla subgradient method, which proceeds according to
\begin{align}
\begin{split}
\bL_{t+1} & \coloneqq\bL_{t}-\eta_{t}\bS_{t}\bR_{t}, \\
\bR_{t+1} & \coloneqq\bR_{t}-\eta_{t}\bS_{t}^{\top}\bL_{t},
\end{split}\label{eq:vanillaSM}
\end{align}
the update rule \eqref{eq:scaledSM} scales the subgradient $\bS_{t}\bR_{t}$ and $\bS_{t}^{\top}\bL_{t}$ by $(\bR_{t}^{\top}\bR_{t})^{-1}$ and $(\bL_{t}^{\top}\bL_{t})^{-1}$, respectively; see \cite{tong2020accelerating} for its counterpart in smooth problems. An important highlight of the scaled subgradient method is that the update rule is covariant with respect to the ambiguity of low-rank matrix factorization. To see this, imagine that we modify the $t$-th updates as 
\begin{align}
\tilde{\bL}_{t}\coloneqq\bL_{t}\bQ, \qquad \tilde{\bR}_{t}\coloneqq\bR_{t}\bQ^{-\top} \label{eq:transform}
\end{align}
for some invertible matrix $\bQ\in\GL(r)$. It is easy to check: 
\begin{enumerate}
\item[\textit{(i)}] both the Polyak's stepsize \eqref{eq:Polyak_stepsize} and the geometrically decaying stepsize \eqref{eq:geometric_stepsize} do not change, since
\begin{align*}
\|\bS_{t}\bR_{t}(\bR_{t}^{\top}\bR_{t})^{-1/2}\|_{\fro}^{2}=\langle\bS_{t}, \bS_{t}\bR_{t}(\bR_{t}^{\top}\bR_{t})^{-1}\bR_{t}^{\top}\rangle = \langle\bS_{t}, \bS_{t}\tilde{\bR}_{t}(\tilde{\bR}_{t}^{\top}\tilde{\bR}_{t})^{-1}\tilde{\bR}_{t}^{\top}\rangle  =\|\bS_{t}\tilde{\bR}_{t}(\tilde{\bR}_{t}^{\top}\tilde{\bR}_{t})^{-1/2}\|_{\fro}^{2},
\end{align*}
which holds similarly for $\|\bS_{t}^{\top}\bL_{t}(\bL_{t}^{\top}\bL_{t})^{-1/2}\|_{\fro}^{2}$;
\item[\textit{(ii)}] The next $(t+1)$-th iterate can be written as
\begin{align*}
\tilde{\bL}_{t+1} = \tilde{\bL}_{t}-\eta_{t}\bS_{t} \tilde{\bR}_{t}(\tilde{\bR}_{t}^{\top}\tilde{\bR}_{t})^{-1} = \left[ \bL_{t}-\eta_{t}\bS_{t}\bR_{t}(\bR_{t}^{\top}\bR_{t})^{-1} \right] \bQ = \bL_{t+1}\bQ,
\end{align*}
and similarly $\tilde{\bR}_{t+1}=\bR_{t+1}\bQ^{-\top}$. Therefore, all the iterates are covariant with respect to the invertible transform \eqref{eq:transform}.
\end{enumerate}

\begin{remark}[Comparison with \texttt{ScaledGD}] Although not our focus, it is instructive to consider the resulting update rule using the nonsmooth $\ell_2$-loss function $f(\bX)=\|\cA(\bX)-\by\|_{2}$ (which has been studied in \cite{charisopoulos2019low}), whose subgradient is given by
\begin{align*}
\bS_{t} = \frac{\cA^{*}(\br_{t})}{\|\br_{t}\|_{2}},
\end{align*}
where $\cA^*(\cdot)$ is the adjoint operator of $\cA(\cdot)$, and $\br_{t}\coloneqq\cA(\bL_{t}\bR_{t}^{\top})-\by $ is the residual using the $t$-th iterate. Consequently, the scaled subgradient method follows the update rule
\begin{align*}
\bL_{t+1} & =\bL_{t}-\frac{\eta_{t}}{\|\br_{t}\|_{2}} \cA^{*}(\br_t)\bR_{t}(\bR_{t}^{\top}\bR_{t})^{-1}, \\
\bR_{t+1} & =\bR_{t}-\frac{\eta_{t}}{\|\br_{t}\|_{2}} \cA^{*}(\br_t)^{\top}\bL_{t}(\bL_{t}^{\top}\bL_{t})^{-1}, 
\end{align*}
for some stepsize $\eta_t$. Careful readers might realize that this coincides with the update rule of \texttt{ScaledGD} in \cite{tong2020accelerating} when optimizing the smooth squared $\ell_2$-loss function $g(\bX)=\frac{1}{2}\|\cA(\bX)-\by\|_{2}^2$, except with an adaptive stepsize $\frac{\eta_{t}}{\|\br_{t}\|_{2}}$. Under the same assumption on $\cA(\cdot)$ in \cite{tong2020accelerating}, the convergence behaviors of \texttt{ScaledSM} applied on $f(\bX)$ match that of \texttt{ScaledGD} on $g(\bX)$. 
\end{remark}

\begin{remark}[\texttt{ScaledSM} for PSD matrices] When the low-rank matrix of interest is positive semi-definite (PSD), we factorize the matrix $\bX\in\RR^{n\times n}$ as $\bX=\bL\bL^{\top}$, with $\bL\in\RR^{n\times r}$. The update rule of \texttt{ScaledSM} simplifies to
\begin{align*}
\bL_{t+1} = \bL_{t} - \eta_{t}\bS_{t}\bL_{t}(\bL_{t}^{\top}\bL_{t})^{-1},
\end{align*}
where $\bS_{t}\in\partial f(\bL_{t}\bL_{t}^{\top})$ is a subgradient of $f(\cdot)$ at $\bL_{t}\bL_{t}^{\top}$. Our theory applies to this PSD case in a straightforward manner.
\end{remark}

\section{Theoretical Guarantees}\label{sec:main-results}

In this section, we first provide the theoretical guarantees of the scaled subgradient method under general geometric assumptions on $f(\cdot)$, and then instantiate them to concrete problems including robust low-rank matrix sensing and quadratic sampling.

\subsection{Geometric assumptions}

We start by introducing the following geometric properties of the loss function $f(\cdot)$, which play a key role in the convergence analysis.

The first condition is similar to the usual Lipschitz property of a function. 
\begin{definition}[{Restricted Lipschitz continuity}] \label{def:restricted_Lipschitz} A function $f(\cdot):\RR^{n_{1}\times n_{2}}\mapsto \RR$ is said to be rank-$r$ restricted $L$-Lipschitz continuous for some quantity $L > 0$ if
\begin{align*}
|f(\bX_{1})-f(\bX_{2})|\le L\|\bX_{1}-\bX_{2}\|_{\fro}
\end{align*}
holds for any $\bX_{1},\bX_{2}\in\RR^{n_{1}\times n_{2}}$ such that $\bX_{1}-\bX_{2}$ has rank at most $2r$.
\end{definition}

The second geometric condition is akin to the (one-point) strong convexity of a function, with the key difference that strong convexity adopts the squared Euclidean norm whereas the following one uses the plain Euclidean norm. 
\begin{definition}[{Restricted sharpness}] \label{def:restricted_sharpness}  A function $f(\cdot):\RR^{n_{1}\times n_{2}}\mapsto\RR$ is said to be rank-$r$ restricted $\mu$-sharp w.r.t.~$\bX_{\star}$ for some $\mu > 0$ if
\begin{align*}
f(\bX)-f(\bX_{\star})\ge\mu\|\bX-\bX_{\star}\|_{\fro}
\end{align*}
holds for any $\bX\in\RR^{n_{1}\times n_{2}}$ with rank at most $r$.
\end{definition}

For notational simplicity, if a function $f$ is both restricted $L$-Lipschitz continuous and $\mu$-sharp, we denote 
\begin{align}
\chi_{f} \coloneqq {L} / {\mu}.\label{eq:def_chif}
\end{align}

In some cases, e.g.~in the presence of noise, the loss function $f(\cdot)$ only satisfies an approximate restricted sharpness property, which is detailed below.

\begin{definition}[{Approximate restricted sharpness}] \label{def:approx_sharpness}  A function $f(\cdot):\RR^{n_{1}\times n_{2}} \mapsto \RR$ is said to be $\xi$-approximate rank-$r$ restricted $\mu$-sharp for some $\mu, \xi>0$ if
\begin{align*}
f(\bX)-f(\bX_{\star})\ge\mu\|\bX-\bX_{\star}\|_{\fro}-\xi
\end{align*}
holds for any $\bX\in\RR^{n_{1}\times n_{2}}$ with rank at most $r$. 
\end{definition}

As shall be seen in Section~\ref{subsec:case_studies}, these conditions can be ensured for proper choices of the loss function as long as the observation operator $\cA(\cdot)$ satisfies certain mixed-norm RIP, which holds for a wide number of practical problems.

\subsection{Main results}

Motivated by \cite{tong2020accelerating}, we measure the performance of $\bF=[\bL^{\top},\bR^{\top}]^{\top}$ using the following error metric
\begin{align}
\dist^{2}(\bF,\bF_{\star})\coloneqq\inf_{\bQ\in\GL(r)}\;\left\Vert (\bL\bQ-\bL_{\star})\bSigma_{\star}^{1/2}\right\Vert _{\fro}^{2}+\left\Vert (\bR\bQ^{-\top}-\bR_{\star})\bSigma_{\star}^{1/2}\right\Vert _{\fro}^{2},\label{eq:dist}
\end{align}
which takes into consideration both the representational ambiguity of the factorization up to invertible transforms and the scaling effect of preconditioners. In comparison, the more standard distance metric \cite{ma2021beyond} in the analysis of vanilla gradient methods reads as follows
\begin{align*}
\dist_{\mathrm{u}}^{2}(\bF,\bF_{\star})\coloneqq\inf_{\bQ\in\GL(r)}\;\left\Vert \bL\bQ-\bL_{\star}\right\Vert _{\fro}^{2}+\left\Vert \bR\bQ^{-\top}-\bR_{\star} \right\Vert _{\fro}^{2}, 
\end{align*}
which is inadequate to delineate the power of preconditioning. See \cite{tong2020accelerating} for more discussions.

We start with stating the linear convergence of the scaled subgradient method when $f(\cdot)$ satisfies both the rank-$r$ restricted $L$-Lipschitz continuity and $\mu$-sharpness. The proof is deferred to Appendix~\ref{proof:theorem_scaledSM}. 

\begin{theorem}[Scaled subgradient method with exact convergence] \label{thm:scaledSM} Suppose that $f(\bX):\RR^{n_{1}\times n_{2}}\mapsto\RR$ is convex in $\bX$, and satisfies rank-$r$ restricted $L$-Lipschitz continuity and $\mu$-sharpness (cf.~Definitions~\ref{def:restricted_Lipschitz} and \ref{def:restricted_sharpness}).
In addition, suppose that the initialization $\bF_{0}$ satisfies 
\begin{align}
\dist(\bF_{0},\bF_{\star})\le0.02\sigma_{r}(\bX_{\star})/\chi_{f},\label{eq:scaledSM_initialization}
\end{align} 
and the scaled subgradient method in \eqref{eq:scaledSM} adopts either Polyak's stepsizes in \eqref{eq:Polyak_stepsize} or geometrically decaying stepsizes in \eqref{eq:geometric_stepsize} with $\lambda=\sqrt{\frac{\sqrt{2}-1}{2}}0.02\sigma_{r}(\bX_{\star})/\chi_{f}^{2}$ and $q=\sqrt{1-0.16/\chi_{f}^{2}}$.
Then for all $t\ge0$, the iterates satisfy
\begin{align*}
\dist(\bF_{t},\bF_{\star})& \le(1-0.16/\chi_{f}^{2})^{t/2}0.02\sigma_{r}(\bX_{\star})/\chi_{f},\quad\mbox{and}\quad \\
\left\Vert \bL_{t}\bR_{t}^{\top}-\bX_{\star}\right\Vert _{\fro}& \le(1-0.16/\chi_{f}^{2})^{t/2}0.03\sigma_{r}(\bX_{\star})/\chi_{f}.
\end{align*}  
\end{theorem}

Theorem~\ref{thm:scaledSM} shows that the iterates of the scaled subgradient method converges at a linear rate; to reach $\epsilon$-accuracy, i.e.~$\|\bL_{t}\bR_{t}^{\top}-\bX_{\star}\|_{\fro}\le\epsilon\sigma_r(\bX_\star)$, it takes at most $O(\chi_{f}^2 \log \frac{1}{\epsilon})$ iterations, which, importantly, is independent of the condition number $\kappa$ of $\bX_{\star}$. In addition, it is still possible to ensure approximate reconstruction when only the approximate restricted sharpness property holds, as shown in the next theorem. Again, we postpone the proof to Appendix~\ref{proof:theorem_scaledSM_noisy}.

\begin{theorem}[Scaled subgradient method with approximate convergence] \label{thm:scaledSM_noisy} Suppose that $f(\cdot):\RR^{n_{1}\times n_{2}}\mapsto\RR$ is convex, and satisfies rank-$r$ restricted $L$-Lipschitz continuity and $\xi$-approximate $\mu$-sharpness (cf.~Definitions~\ref{def:restricted_Lipschitz}~and~\ref{def:approx_sharpness}) for some $\xi\le10^{-3}\sigma_{r}(\bX_{\star})\mu/\chi_{f}$. Suppose that the initialization $\bF_{0}$ satisfies $\dist(\bF_{0},\bF_{\star})\le0.02\sigma_{r}(\bX_{\star})/\chi_{f}$, and the scaled subgradient method adopts geometrically decaying stepsizes \eqref{eq:geometric_stepsize} with $\lambda=\sqrt{\frac{\sqrt{2}-1}{2}}0.02\sigma_{r}(\bX_{\star})/\chi_{f}^{2}$ and $q=\sqrt{1-0.13/\chi_{f}^{2}}$.
Then for all $t\ge0$, the iterates satisfy
\begin{align*}
\dist(\bF_{t},\bF_{\star})& \le \max\left\{(1-0.13/\chi_{f}^{2})^{t/2}0.02\sigma_{r}(\bX_{\star})/\chi_{f}, 20\xi/\mu\right\},\quad\mbox{and} \\
\left\Vert \bL_{t}\bR_{t}^{\top}-\bX_{\star}\right\Vert _{\fro}& \le \max\left\{(1-0.13/\chi_{f}^{2})^{t/2}0.03\sigma_{r}(\bX_{\star})/\chi_{f}, 30\xi/\mu\right\}.
\end{align*} 
\end{theorem}

Theorem~\ref{thm:scaledSM_noisy} shows that as long as the relaxation parameter $\xi$ is sufficiently small, i.e.~$\xi\lesssim \sigma_{r}(\bX_{\star})\mu/\chi_{f}$, then the scaled subgradient method with geometrically decaying stepsizes converges at a linear rate until an error floor is hit. In particular, the iterates satisfy $\|\bL_{t}\bR_{t}^{\top}-\bX_{\star}\|_{\fro} \le 30\xi/\mu$ after at most $O(\chi_f^2)$ iterations up to logarithmic factors.

\begin{remark}
For simplicity of exposition, we have fixed the values of $\lambda$ and $q$ for the geometrically decaying stepsizes in the above theorems. It is possible to allow a wider range of $\lambda$ and $q$ by slightly modifying the arguments without sacrificing the linear convergence. In practice, these parameters should be tuned in order to yield optimal performance. 
\end{remark}

\subsection{A case study: robust low-rank matrix recovery} \label{subsec:case_studies}

We now apply the above theorems to robust low-rank matrix recovery, which showcases the superior performance of the scaled subgradient method. 

\paragraph{Noise-free case.} We start with the observation model \eqref{eq:observation_model} with clean measurements, i.e.~$\bw=\zero$ and $\bs=\zero$. To proceed, we assume that the observation operator $\cA(\cdot)$ satisfies the following mixed-norm RIP.
\begin{definition}[{mixed-norm RIP \cite{recht2010guaranteed,chen2015exact,charisopoulos2019low}}] \label{def:mixed_rip} The linear map $\cA(\cdot)$ is said to obey the rank-$2r$ mixed-norm RIP with constants $\delta_{1},\delta_{2}$ if for all matrices $\bM\in\RR^{n_{1}\times n_{2}}$ of rank at most $2r$, one has 
\begin{align*}
\delta_{1}\|\bM\|_{\fro} \le \|\cA(\bM)\|_1 \le \delta_{2}\|\bM\|_{\fro}.
\end{align*}
\end{definition}
The next proposition verifies that the loss function \eqref{eq:loss} satisfies restricted Lipschitz continuity and sharpness properties under the mixed-norm RIP.

\begin{proposition}\label{prop:mixedRIP_clean} If $\cA(\cdot)$ satisfies rank-$2r$ mixed-norm RIP with constants $(\delta_{1},\delta_{2})$, then $f(\bX) = \|\cA(\bX)-\by\|_1= \|\cA(\bX-\bX_{\star})\|_1$ in \eqref{eq:loss} satisfies the rank-$r$ restricted $L$-Lipschitz continuity and $\mu$-sharpness with
\begin{align*}
L=\delta_{2},\quad\text{and}\quad\mu=\delta_{1}.
\end{align*}
\end{proposition}
\begin{proof}
See Appendix~\ref{proof:prop_mixedRIP_clean}.
\end{proof}

With the geometric characterization of $f(\cdot)$ in place, we immediately have the following corollary that captures the performance of the scaled subgradient method when $\cA(\cdot)$ satisfies the mixed-norm RIP.
\begin{corollary}
If $\cA(\cdot)$ satisfies rank-$2r$ mixed-norm RIP with $(\delta_{1},\delta_{2})$, then the scaled subgradient method over the loss function $f(\bX)=\|\cA(\bX)-\by\|_1$ using either Polyak's or geometrically decaying stepsizes achieves $\left\Vert \bL_{t}\bR_{t}^{\top}-\bX_{\star}\right\Vert _{\fro}\le \epsilon \sigma_r(\bX_\star)$ in $O\left(\frac{\delta_2^2}{\delta_1^2} \log \frac{1}{\epsilon} \right)$ iterations as long as the initialization satisfies $\dist(\bF_{0},\bF_{\star})\le \frac{0.02\delta_1}{\delta_2}\sigma_{r}(\bX_{\star})$.
\end{corollary}

\paragraph{Noisy and corrupted case.} We now consider the observation model \eqref{eq:observation_model} where the noise $\bw$ is bounded with $\|\bw\|_1\le \sigma_w$ and $\|\bs\|_{0}=p_{s}m$, where $p_{s} \in [0,1/2)$ is the fraction of outliers. Following \cite{charisopoulos2019low}, we further introduce another important property of $\cA(\cdot)$.
\begin{definition}[{$\cS$-outlier bound \cite{charisopoulos2019low}}] \label{def:s_outlier} The linear map $\cA(\cdot)$ is said to obey the rank-$2r$ $\cS$-outlier bound w.r.t.~a set $\cS$ with a constant $\delta_{3}$ if for all matrices $\bM\in\RR^{n_{1}\times n_{2}}$ of rank at most $2r$, one has 
\begin{align*}
\delta_{3}\|\bM\|_{\fro}\le \|\cA_{\cS^c}(\bM)\|_1 -\|\cA_{\cS}(\bM)\|_1,
\end{align*}
where $\cA_{\cS}(\bM) =\{\cA_i(\bM)\}_{i\in \cS}$ and $\cA_{\cS^c}(\bM) =\{\cA_i(\bM)\}_{i\in \cS^c}$. 
\end{definition}

The next proposition verifies that the loss function in \eqref{eq:loss} satisfies restricted Lipschitz continuity and approximate sharpness properties under the mixed-norm RIP (cf.~Definition~\ref{def:mixed_rip}) and the $\cS$-outlier bound (cf.~Definition~\ref{def:s_outlier}).

\begin{proposition}[Matrix sensing with outliers] \label{prop:mixedRIP_corrupted} Denote the support of the outlier $\bs$ as $\cS$. Suppose that $\cA(\cdot)$ satisfies rank-$2r$ mixed-norm RIP with $(\delta_{1},\delta_{2})$ and $\cS$-outlier bound with $\delta_{3}$, then $f(\bX)$ in \eqref{eq:loss} satisfies rank-$r$ restricted $L$-Lipschitz continuity and $\xi$-approximate $\mu$-sharpness with 
\begin{align}
L=\delta_{2},\quad \mu=\delta_{3}, \quad\mbox{and}\quad \xi=2\sigma_w.
\end{align}
\end{proposition}
\begin{proof} 
See Appendix~\ref{proof:prop_mixedRIP_corrupted}.
\end{proof}

Similar to the previous noise-free case, this immediately leads to performance guarantees of the scaled subgradient method when $\cA(\cdot)$ satisfies both the mixed-norm RIP and the $\cS$-outlier bound.
\begin{corollary}
If $\cA(\cdot)$ satisfies rank-$2r$ mixed-norm RIP with $(\delta_{1},\delta_{2})$ and $\cS$-outlier bound with $\delta_3$, and $\|\bw\|_1 \le \sigma_w \le 10^{-3} \sigma_r(\bX_{\star})\delta_3^2/\delta_2$, then the scaled subgradient method over the loss function $f(\bX) = \|\cA(\bX)-\by\|_1$ using the geometrically decaying stepsizes achieves $\left\Vert \bL_{t}\bR_{t}^{\top}-\bX_{\star}\right\Vert _{\fro} \le \max\left\{\epsilon\sigma_{r}(\bX_{\star}), 60\sigma_w/\delta_3\right\}$ in $O\left(\frac{\delta_2^2}{\delta_3^2} \log \frac{1}{\epsilon} \right)$ iterations as long as the initialization satisfies $\dist(\bF_{0},\bF_{\star})\le \frac{0.02\delta_3}{\delta_2}\sigma_{r}(\bX_{\star})$.
\end{corollary}
 
We now instantiate the above general guarantee to the following two types of observation operators. For simplicity, we assume there is no dense noise, i.e.~$\sigma_w=0$; see Table~\ref{tab:Performance-guarantees} for a summary of the comparisons.
\begin{itemize}
\item \textit{matrix sensing:} the measurement operator $\cA_i(\cdot)$ is defined as $\cA_i(\bX_{\star}) = \frac{1}{m}\langle \bA_i, \bX_{\star}\rangle$, where the matrix $\bA_i$ is composed of i.i.d.~Gaussian entries $\cN(0,1)$.\footnote{The same guarantee also holds for sub-Gaussian measurements.} It is shown in \cite{charisopoulos2019low} (see also \cite{li2020nonconvex}) that $\cA(\cdot)$ satisfies the mixed-norm RIP and $\cS$-outlier bound with
\begin{align*}
\delta_1 \gtrsim 1, \quad\delta_2 \lesssim 1, \quad\delta_3 \gtrsim 1-2p_s,
\end{align*}
as long as $m\gtrsim \frac{(n_1+n_2)r}{(1-2p_s)^2}\log(\frac{1}{1-2p_s})$. Hence, the scaled subgradient method converges linearly to $\epsilon$-accuracy in $O\left(\frac{1}{(1-2p_s)^2}\log\frac{1}{\epsilon} \right)$ iterations provided that it is initialized properly, making it robust simultaneously to ill-conditioning of the matrix $\bX_{\star}$ and the presence of the outliers.

\item \textit{quadratic sampling:} the measurement operator $\cA_i(\cdot)$ is defined as $\cA_i(\bX_{\star}) = \frac{1}{m}\langle \ba_i\ba_i^{\top}, \bX_{\star}\rangle$, where $\bX_{\star}\in\RR^{n\times n}$ is PSD and the vector $\ba_i$ is composed of i.i.d.~Gaussian entries $\cN(0,1)$. It is shown in \cite{charisopoulos2019low} that $\cA(\cdot)$ satisfies the mixed-norm RIP and $\cS$-outlier bound with
\begin{align*}
\delta_1 \gtrsim 1, \quad\delta_2 \lesssim \sqrt{r}, \quad\delta_3 \gtrsim 1-2p_s,
\end{align*}
as long as $m\gtrsim \frac{nr^2}{(1-2p_s)^2}\log(\frac{\sqrt{r}}{1-2p_s})$. Hence, the scaled subgradient method converges linearly to $\epsilon$-accuracy in $O\left(\frac{r}{(1-2p_s)^2}\log\frac{1}{\epsilon} \right)$ iterations, as long as it is seeded with a good initialization. In comparison, the iteration complexity of the scaled gradient descent method over the least-squares loss function depends polynomially with respect to $n$, due to the heavy-tailed nature of the observation operator, let alone its sensitivity to the outliers.
\end{itemize}

\begin{remark}[Initialization]The above discussions are limited to the local iteration complexity, assuming a good initialization satisfying \eqref{eq:scaledSM_initialization} is available. In the absence of outliers, a standard spectral method can be used, as shown in \cite{tong2020accelerating}. In the presence of outliers, a truncated spectral method could be used; see e.g.~\cite{zhang2016provable,li2020non}.
\end{remark}

\section{Numerical Experiments} \label{sec:numerical}

In this section, we conduct numerical experiments to corroborate our theory. 

\paragraph{Comparisons of \texttt{ScaledSM} and $\texttt{VanillaSM}$.}
Since the vanilla subgradient method (\texttt{VanillaSM}) has been extensively benchmarked against other methods and established as state-of-the-art in \cite{charisopoulos2019low}, we focus on comparing the proposed scaled subgradient method (\texttt{ScaledSM}) to \texttt{VanillaSM} in the sequel. 
In general, the geometrically decaying stepsize \eqref{eq:geometric_stepsize} is a more practical choice than the Polyak's stepsize \eqref{eq:Polyak_stepsize}, especially in the presence of noise and outliers. Nonetheless, using properly tuned geometrically decaying stepsizes essentially matches the performance of using Polyak's stepsizes, for both \texttt{VanillaSM} \cite{li2020nonconvex} and \texttt{ScaledSM}, the latter of which we shall illustrate in the ensuing experiments. As such, we adopt Polyak's stepsizes in the comparisons below, to emulate the scenario where both methods are tuned to operate under its largest allowable stepsizes and achieve the fastest convergence. In addition, both algorithms start from the same initialization.

We consider two low-rank matrix estimation tasks discussed in Section~\ref{subsec:case_studies}. Recall the observation model in \eqref{eq:observation_model} and its entrywise version in \eqref{eq:entrywise_observation}, which we repeat below for convenience:
\begin{align*}
y_i = \cA_i(\bX_{\star}) + w_i + s_i, \qquad 1\le i\le m.
\end{align*}
In both tasks, the noise entry $w_i$ is composed of i.i.d.~entries uniformly drawn from $[-\frac{\sigma_{w}}{m},\frac{\sigma_{w}}{m}]$. The outlier $s_i = \bar{s}_i \Omega_i$ is a sparse vector where $\Omega_i$ is a Bernoulli random variable with probability $p_s \in [0,1/2)$, and $\bar{s}_i$ is drawn uniformly at random from $[-10\|\cA(\bX_{\star})\|_{\infty}, 10\|\cA(\bX_{\star})\|_{\infty}]$.
For ease of presentation, we assume that $\bX_{\star}\in\RR^{n\times n}$ is a square matrix with rank as $r$. We collect $m=8nr$ measurements using the following respective measurement models. The signal-to-noise ratio is defined as $\mathrm{SNR}\coloneqq20\log_{10}\frac{\|\cA(\bX_{\star})\|_{1}}{\sigma_{w}}$ in $\mathrm{dB}$.
\begin{itemize}
\item {\em Matrix sensing.} Here, the measurement operator $\cA_i(\cdot)$ is defined as $\cA_i(\bX_{\star}) = \frac{1}{m}\langle \bA_i, \bX_{\star}\rangle$, where the matrix $\bA_i$ is composed of i.i.d.~Gaussian entries $\cN(0,1)$. The ground truth matrix $\bX_{\star}$ is generated via its compact SVD $\bX_{\star} = \bU_{\star}\bSigma_{\star}\bV_{\star}^{\top}$, where $\bU_{\star}\in\RR^{n\times r}$ is generated as the orthonormal basis vectors of an $n\times r$ matrix with i.i.d.~Rademacher entries, $\bSigma_{\star}$ is a diagonal matrix with the diagonal entries linearly distributed from $1$ to $\kappa$, and $\bV_{\star}\in\RR^{n\times r}$ is generated in a similar fashion to $\bU_{\star}$. 

\item {\em Quadratic sampling.} Here, the measurement operator $\cA_i(\cdot)$ is defined as $\cA_i(\bX_{\star}) = \frac{1}{m}\langle \ba_i\ba_i^{\top}, \bX_{\star} \rangle$, where $\ba_{i}$ is composed of i.i.d.~Gaussian entries $\cN(0,1)$. The ground truth matrix $\bX_{\star}$ is positive semi-definite, and is generated via its compact SVD $\bX_{\star}=\bU_{\star}\bSigma_{\star}\bU_{\star}^{\top}$, where $\bU_{\star}$ and $\bSigma_{\star}$ are generated in the same manner described above. 
\end{itemize}

\begin{figure}[!ht]
\centering
\begin{tabular}{cc}
\includegraphics[width=0.48\linewidth]{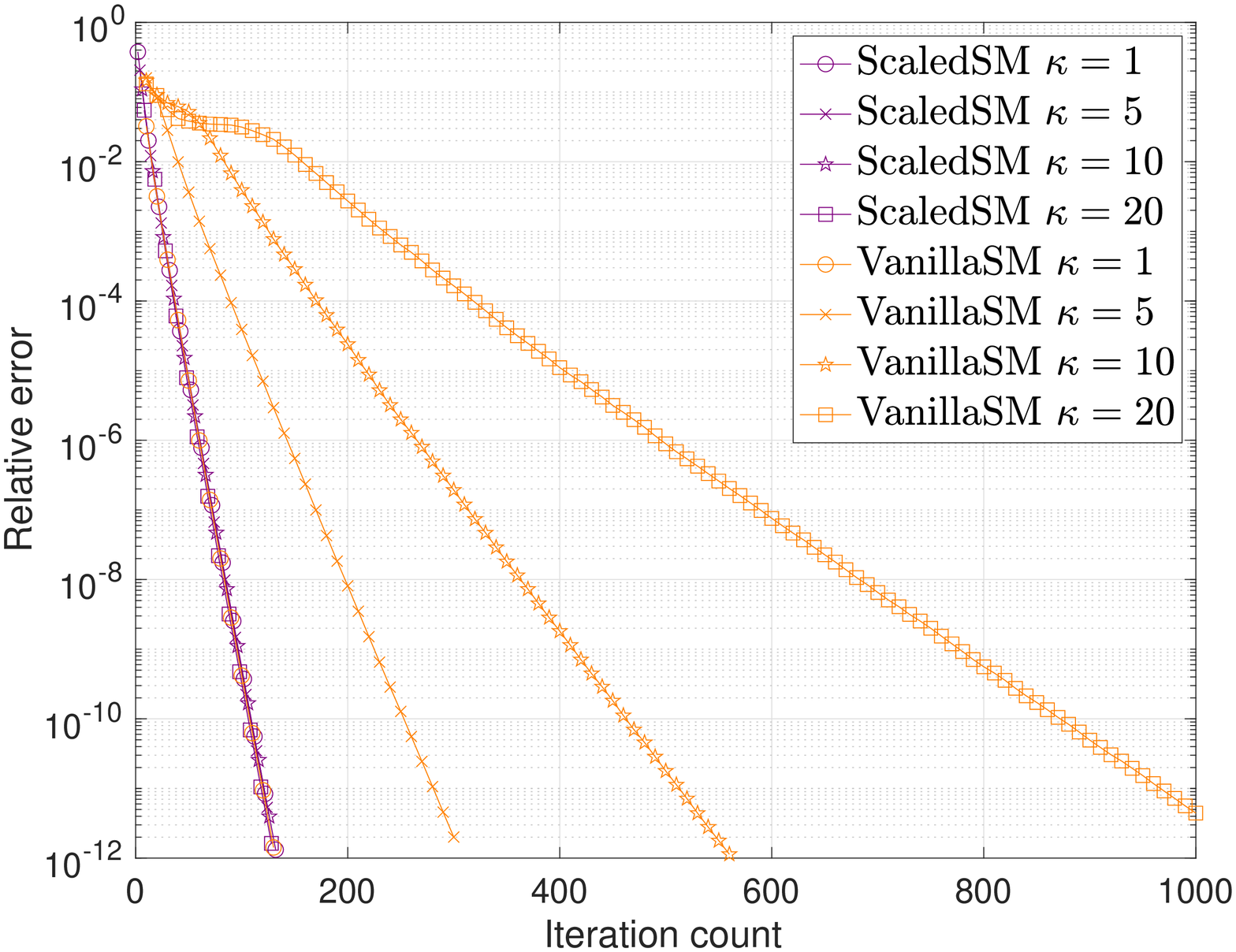} & \includegraphics[width=0.48\linewidth]{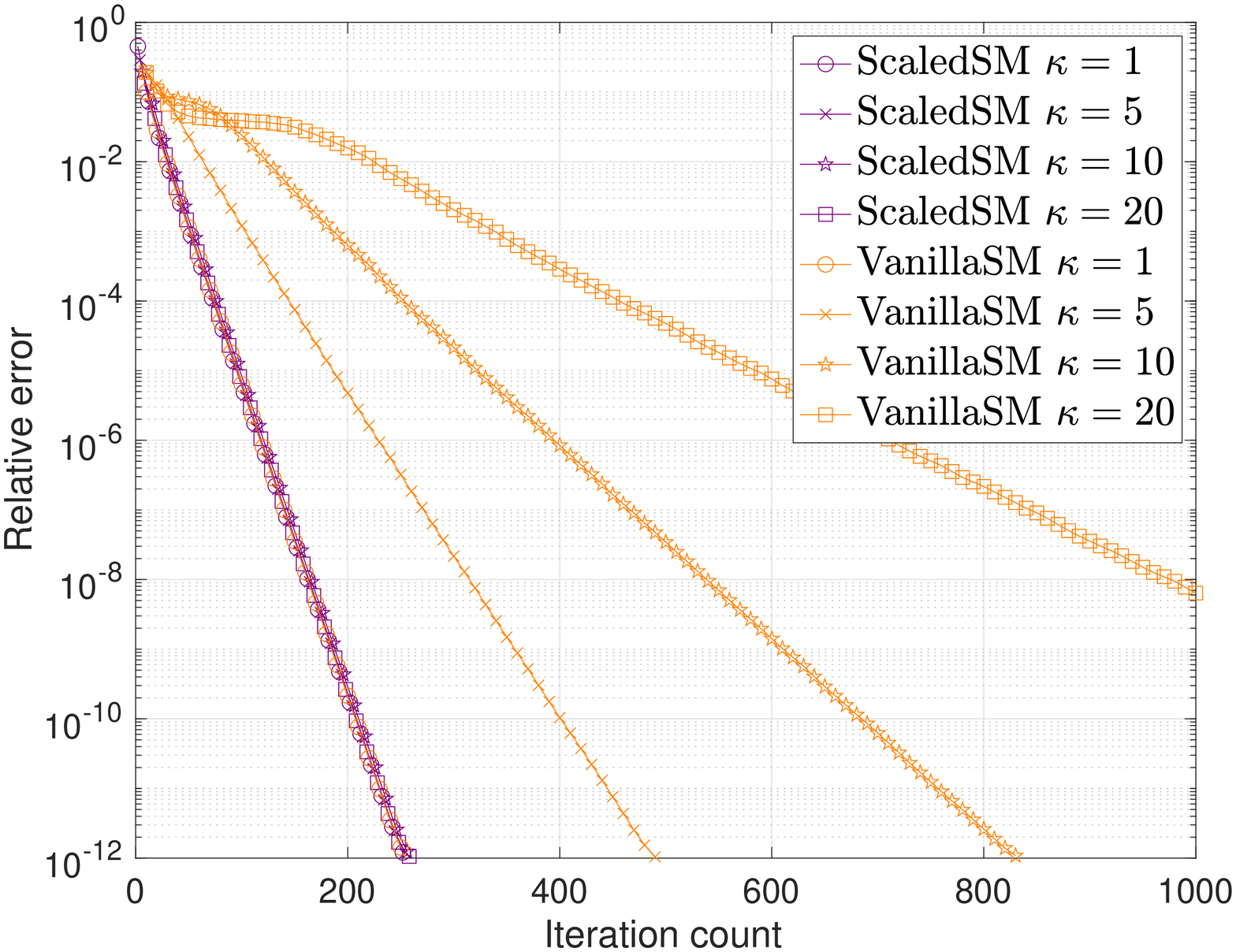} \\
{\small (a) without outliers} & {\small (b) with $20\%$ outliers }
\end{tabular}
\caption{Performance comparisons of  \texttt{ScaledSM}  and \texttt{VanillaSM} for matrix sensing without or with outliers under different condition numbers $\kappa = 1,5,10,20$, where $n=100$, $r=10$, and $m=8nr$.}\label{fig:MS}
\end{figure}

\begin{figure}[!ht]
\centering
\begin{tabular}{cc}
\includegraphics[width=0.48\linewidth]{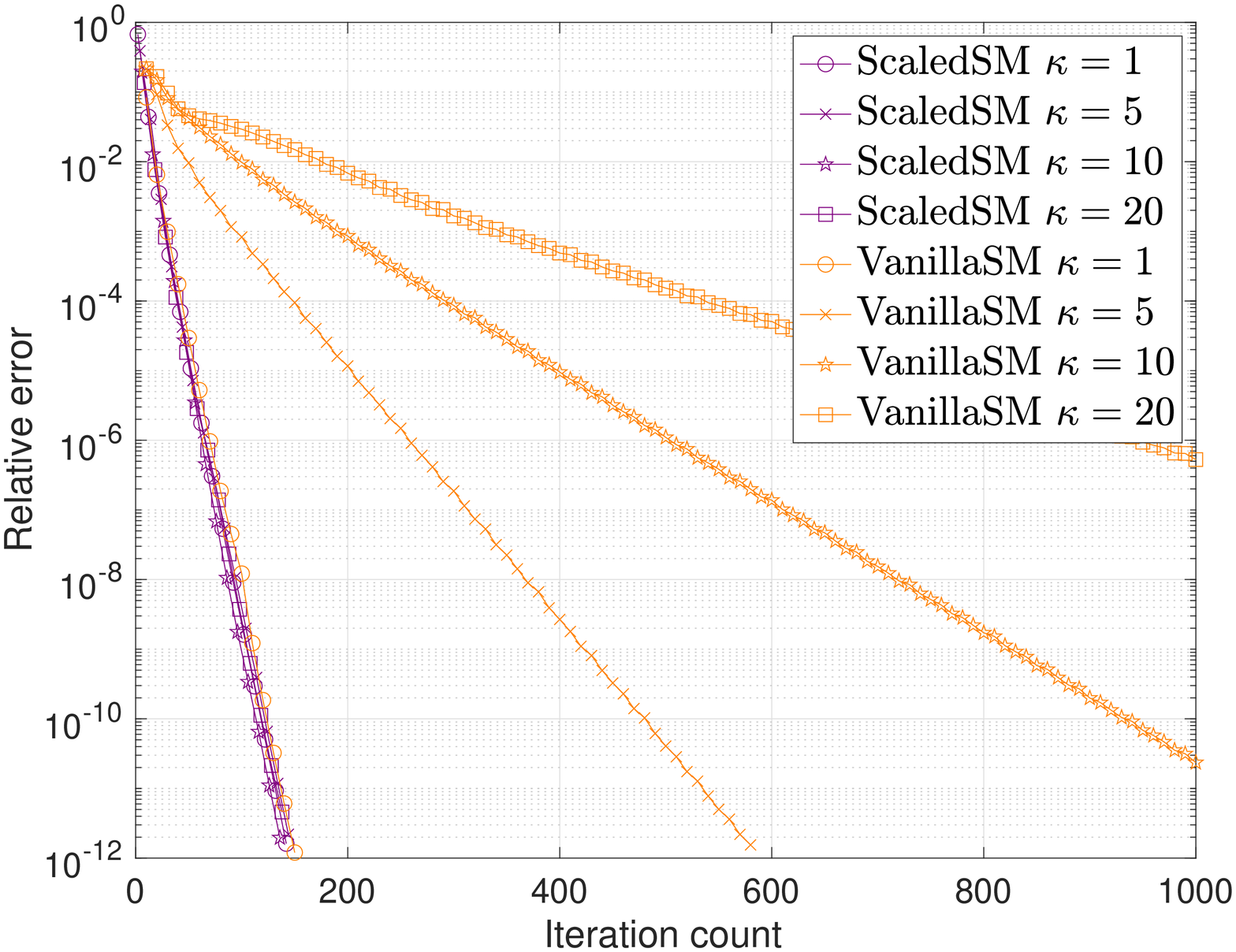} & \includegraphics[width=0.48\linewidth]{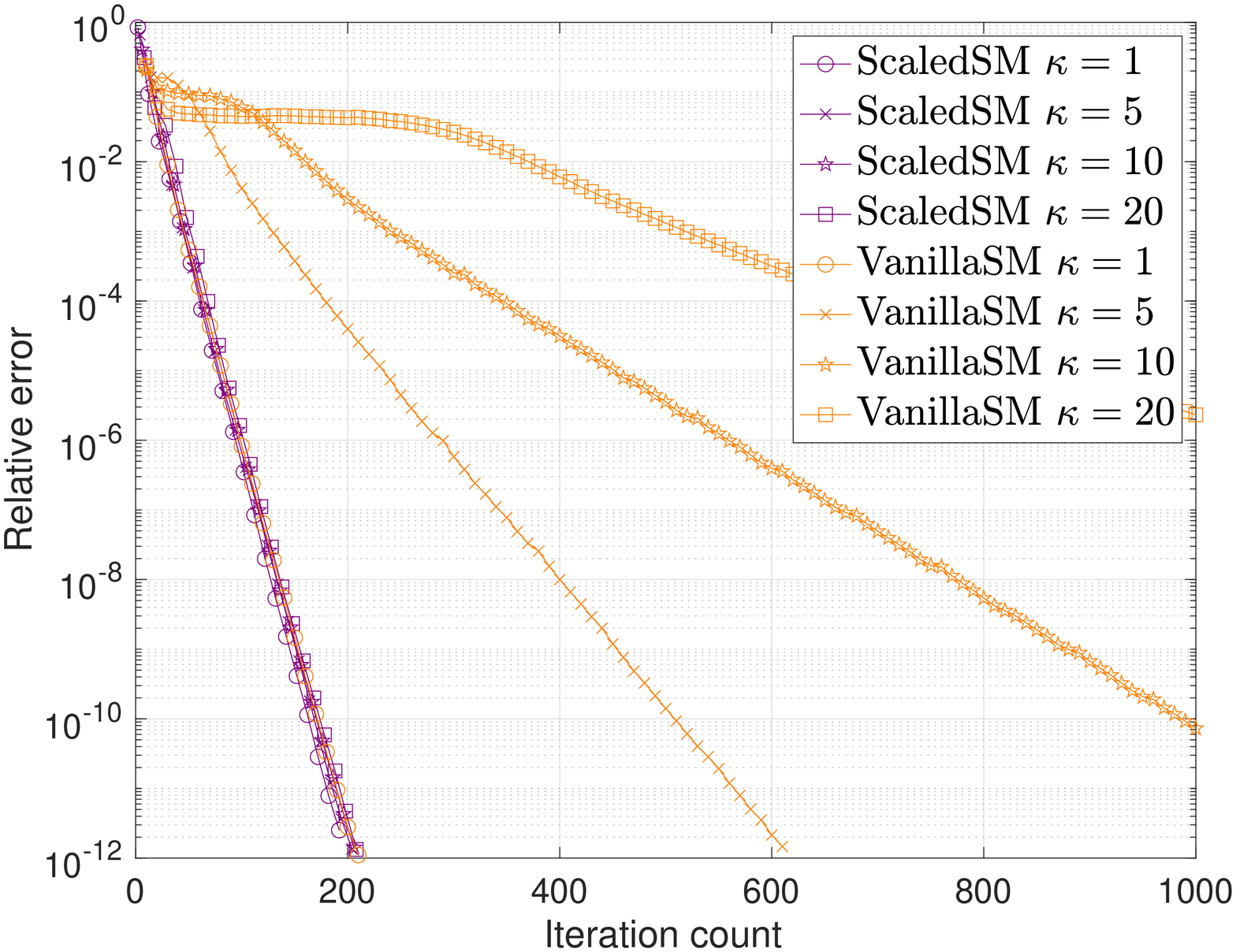} \\
{\small (a) without outliers} & {\small (b) with $20\%$ outliers} 
\end{tabular}
\caption{Performance comparisons of \texttt{ScaledSM} and \texttt{VanillaSM} for quadratic sampling without or with outliers under different condition numbers $\kappa = 1,5,10,20$, where $n=100$, $r=5$, and $m=8nr$.}\label{fig:QS}
\end{figure}

Denote the index set of the remaining measurements after discarding $p_{s}$ fraction with largest amplitudes as $\cI=\{i: |y_i| \le |\by|_{(\lceil p_s m \rceil)}\}$, where $|\by|_{(k)}$ denotes the $k$th largest amplitude of $\by$. The truncated spectral method in \cite{zhang2016provable,li2020non} is used for initialization, where we apply the standard spectral method only on the subset $\cI$ of the measurements. For matrix sensing, it follows the prescription in \cite{li2020non}, and for quadratic sampling, it follows \cite{li2018nonconvex}.

Fig.~\ref{fig:MS} shows the relative reconstruction error $\|\bX_{t}-\bX_{\star}\|_{\fro}/\|\bX_{\star}\|_{\fro}$ for matrix sensing without outliers (in (a)) and with $20\%$ outliers (i.e.~$p_s=0.2$ in (b)) under different condition numbers $\kappa$, where $\bX_{t}$ is the estimated low-rank matrix at the $t$-th iteration. Fig.~\ref{fig:QS} shows the relative reconstruction error for quadratic sampling under the same setting. It can be seen that \texttt{ScaledSM} is insensitive to $\kappa$ and converges as a fast rate that is independent with $\kappa$, while the convergence of \texttt{VanillaSM} slows down dramatically with the increase of $\kappa$. In addition, both algorithms still converge linearly in the presence of outliers, thanks to the robustness of the least absolute deviations.

\begin{figure}[!ht]
\centering
\begin{tabular}{cc}
\includegraphics[width=0.48\linewidth]{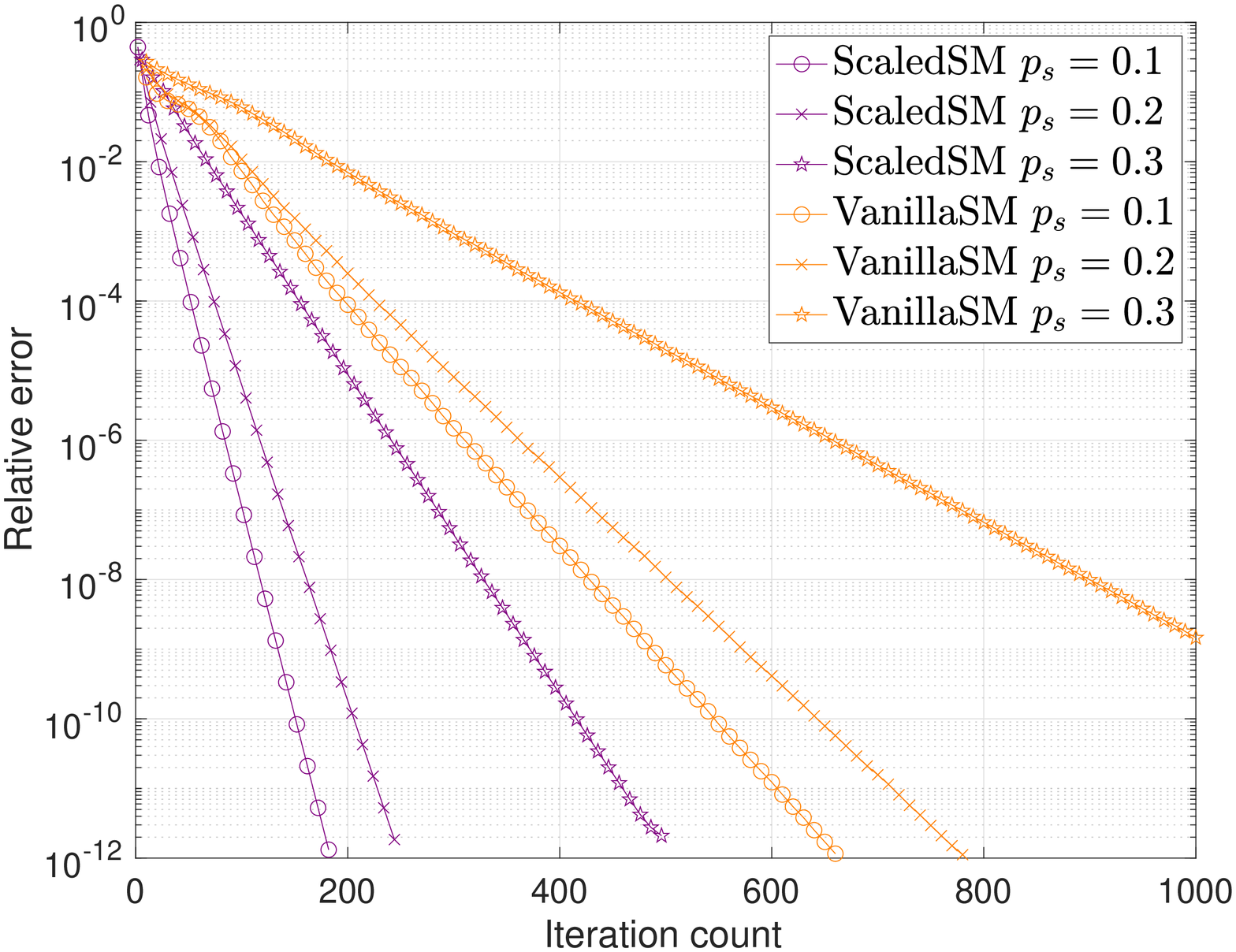}  & \includegraphics[width=0.48\linewidth]{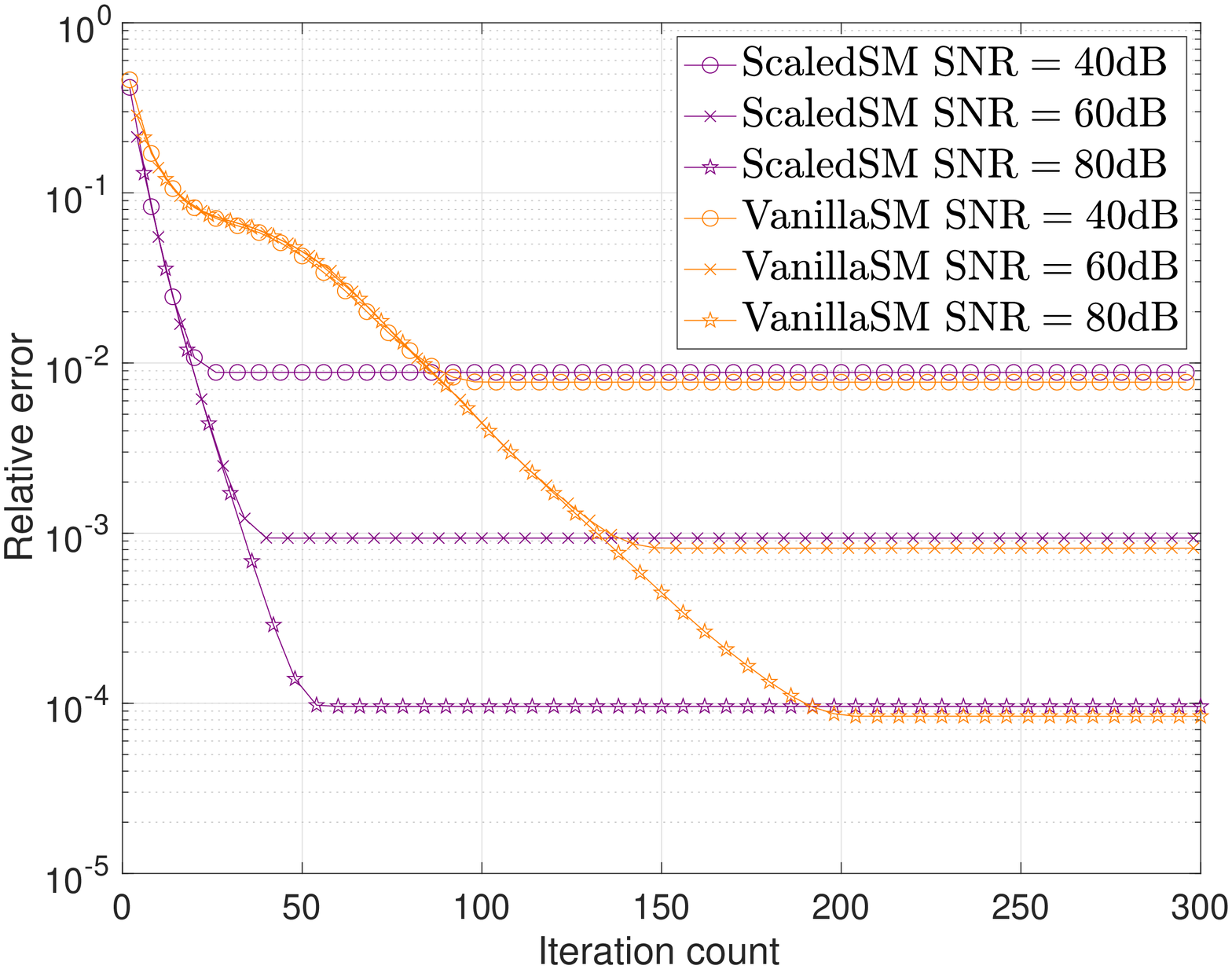}\\
{\small (a) with varying outliers} & {\small (b) with $10\%$ outliers and noise}
\end{tabular}
\caption{Performance comparisons of \texttt{ScaledSM} and \texttt{VanillaSM} for matrix sensing under different noise and outlier models, where $n=100$, $r=10$, $m=8nr$, and $\kappa=10$.}\label{fig:MS_outliers}
\end{figure}

Fig.~\ref{fig:MS_outliers} further examines the impact of the amount of outliers and noise on the convergence speed in matrix sensing with a fixed condition number $\kappa =10$, where Fig.~\ref{fig:MS_outliers}~(a) illustrates the convergence speed at varying amounts of outliers $p_s =0.1, 0.2, 0.3$ respectively, and Fig.~\ref{fig:MS_outliers}~(b) illustrates the convergence with $p_s =0.1$ and additional bounded noise with varying $\mathrm{SNR}=40,60,80\mathrm{dB}$. Similarly, Fig.~\ref{fig:QS_outliers} shows the same plots for quadratic sampling under the same setting. 
It can be seen that the convergence rate of \texttt{ScaledSM} slows down with the increase of outliers, which is again, consistent with the theory. Furthermore, the reconstruction is robust in the presence of additional bounded noise, where both \texttt{ScaledSM} and \texttt{VanillaSM} converge to the same accuracy that is proportional to the noise level, with \texttt{ScaledSM} converging at a faster speed.

\begin{figure}[!ht]
\centering
\begin{tabular}{cc}
\includegraphics[width=0.48\linewidth]{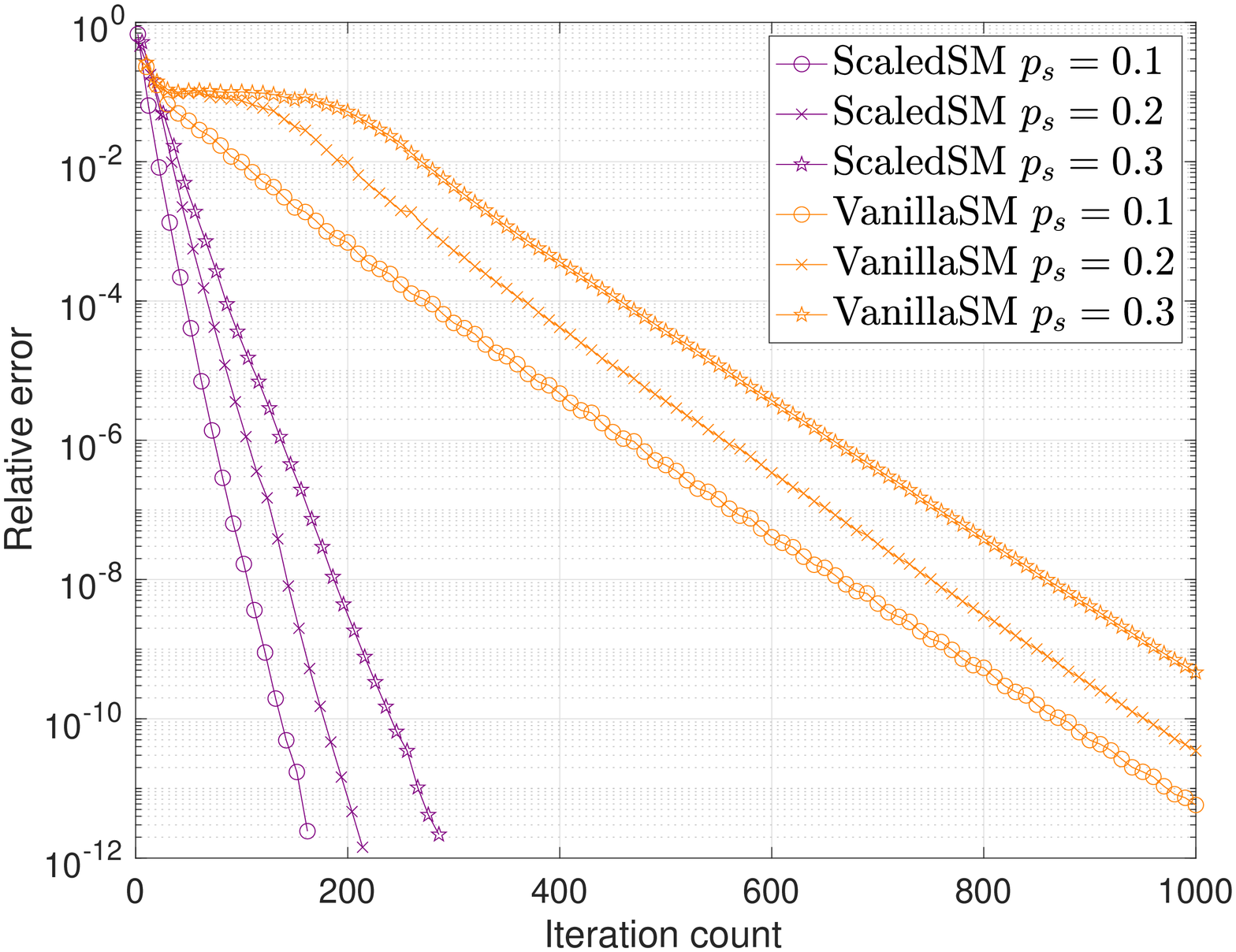}  & \includegraphics[width=0.48\linewidth]{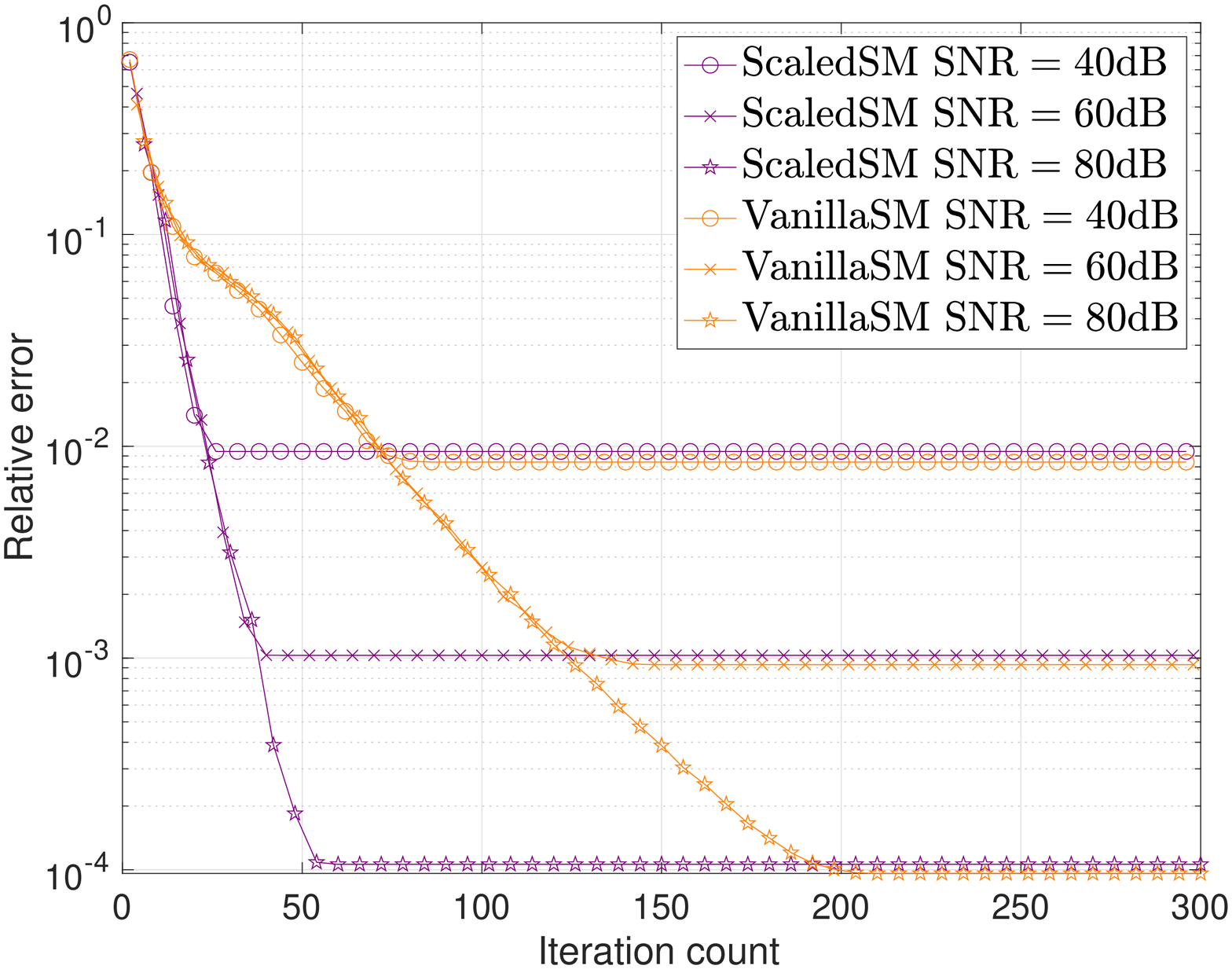}\\
{\small (a) with varying outliers} & {\small (b) with $10\%$ outliers and noises}
\end{tabular}
\caption{Performance comparisons of \texttt{ScaledSM} and \texttt{VanillaSM} for quadratic sampling under different noise and outlier models, where $n=100$, $r=5$, $m=8nr$, and $\kappa=10$. }\label{fig:QS_outliers}
\end{figure}

\paragraph{Comparisons of stepsize schedules.} We now compare the geometrically decaying stepsize with the Polyak's stepsize for \texttt{ScaledSM}, which essentially mirrors similar experiments conducted in \cite{li2020nonconvex} for \texttt{VanillaSM}. We run \texttt{ScaledSM} for at most $T=1000$ iterations, and stop early if the relative error achieves $10^{-12}$. Fig.~\ref{fig:MS_stepsizes} and Fig.~\ref{fig:QS_stepsizes} show the performance comparisons of \texttt{ScaledSM} under various stepsize schedules for matrix sensing and quadratic sampling, respectively. For both figures, (a) shows the final relative error of \texttt{ScaledSM} using geometrically decaying stepsizes under various $(\lambda, q)$, where we see that \texttt{ScaledSM} converges as long as $\lambda$ is not too large and $q$ is not too small. We further plot the relative error versus the iteration count for \texttt{ScaledSM} using geometrically decaying stepsizes with a fixed $q$ and various $\lambda$ in (b), and with a fixed $\lambda$ and various $q$ in (c), where the performance using Polyak's stepsizes is plotted for comparison. It can be seen that using Polyak's stepsizes yields the fastest convergence. Indeed, if properly tuned, geometrically decaying stepsizes match Polyak's stepsizes, as shown in (d). In general, we find that there is a wide range of parameters for geometrically decaying stepsizes where \texttt{ScaledSM} converges in a fast speed comparable to that of using Polyak's stepsizes, as long as $\lambda$ is not too large and $q$ is not too small.

\begin{figure*}[!ht]
\centering
\begin{tabular}{cc}
\includegraphics[width=0.48\textwidth]{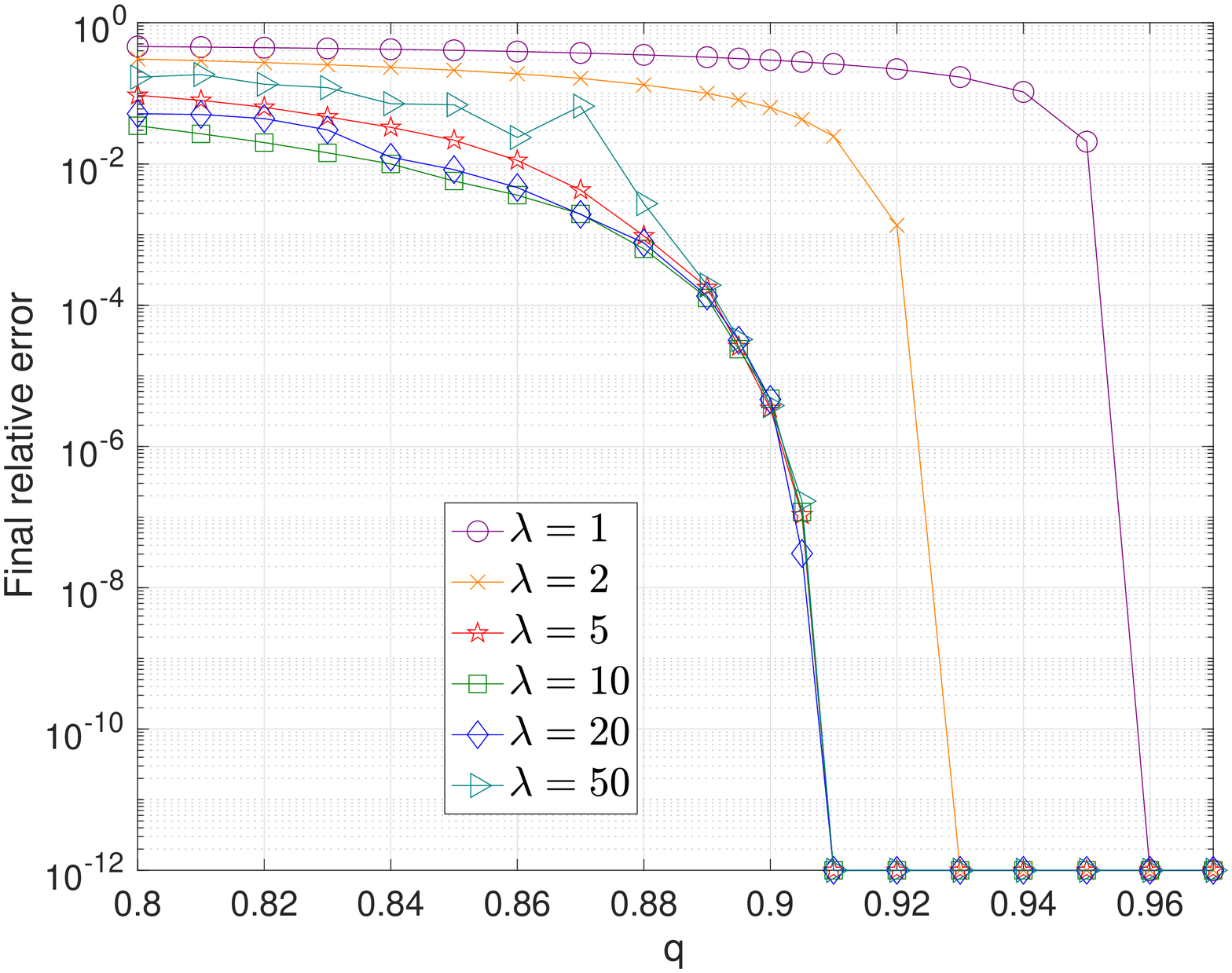} &
\includegraphics[width=0.48\textwidth]{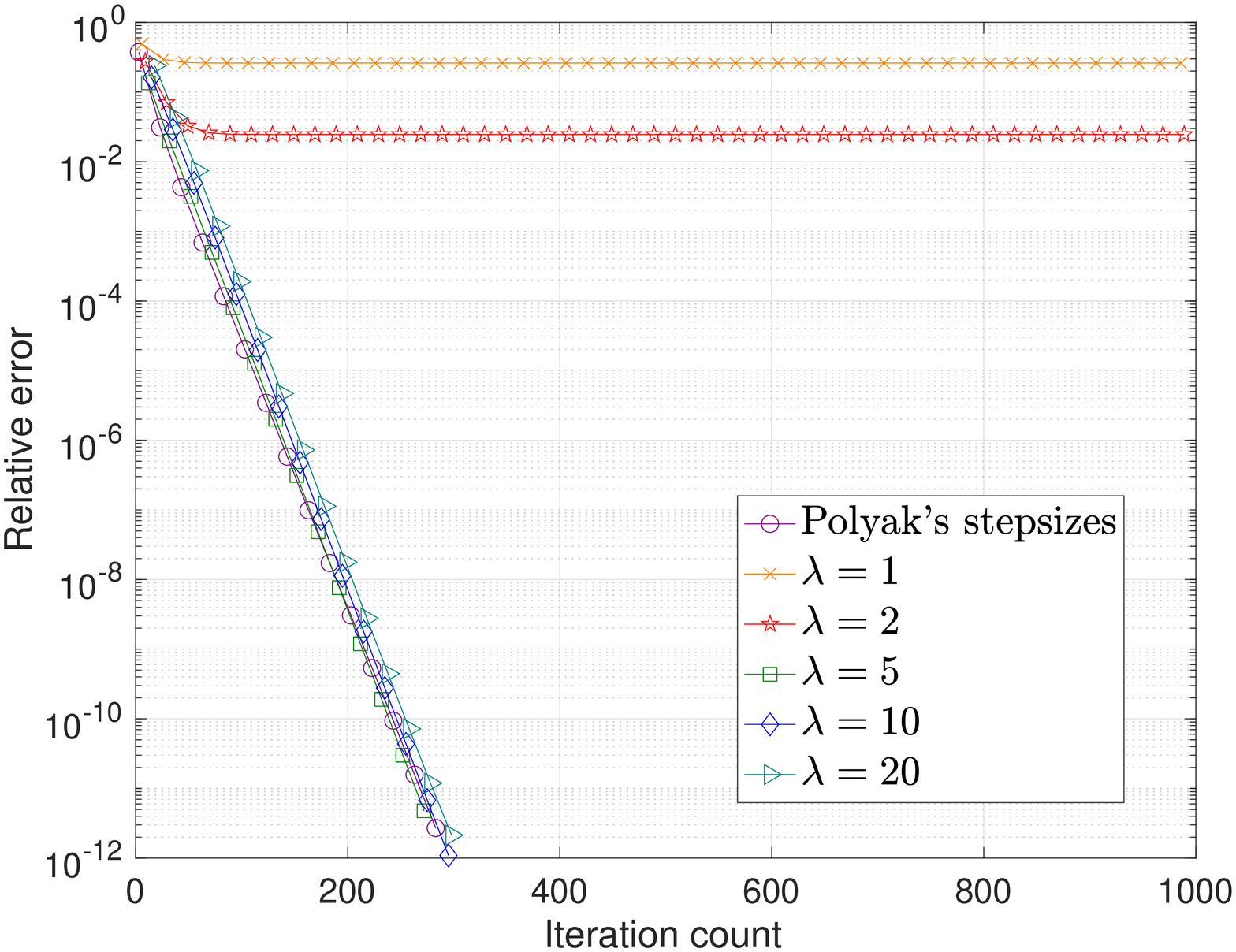} \\
{\small (a)} & {\small (b)} \\
\includegraphics[width=0.48\textwidth]{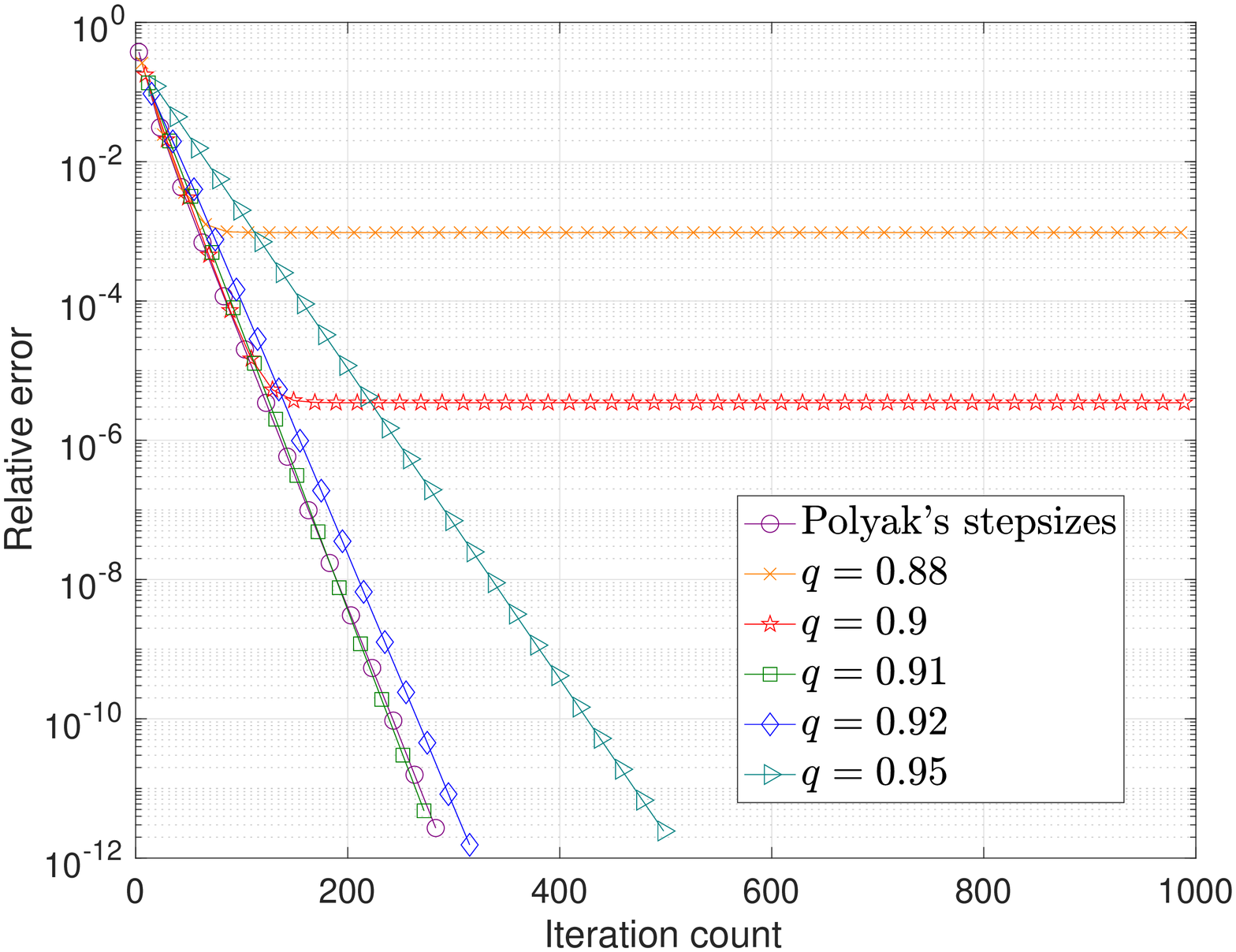} &
\includegraphics[width=0.48\textwidth]{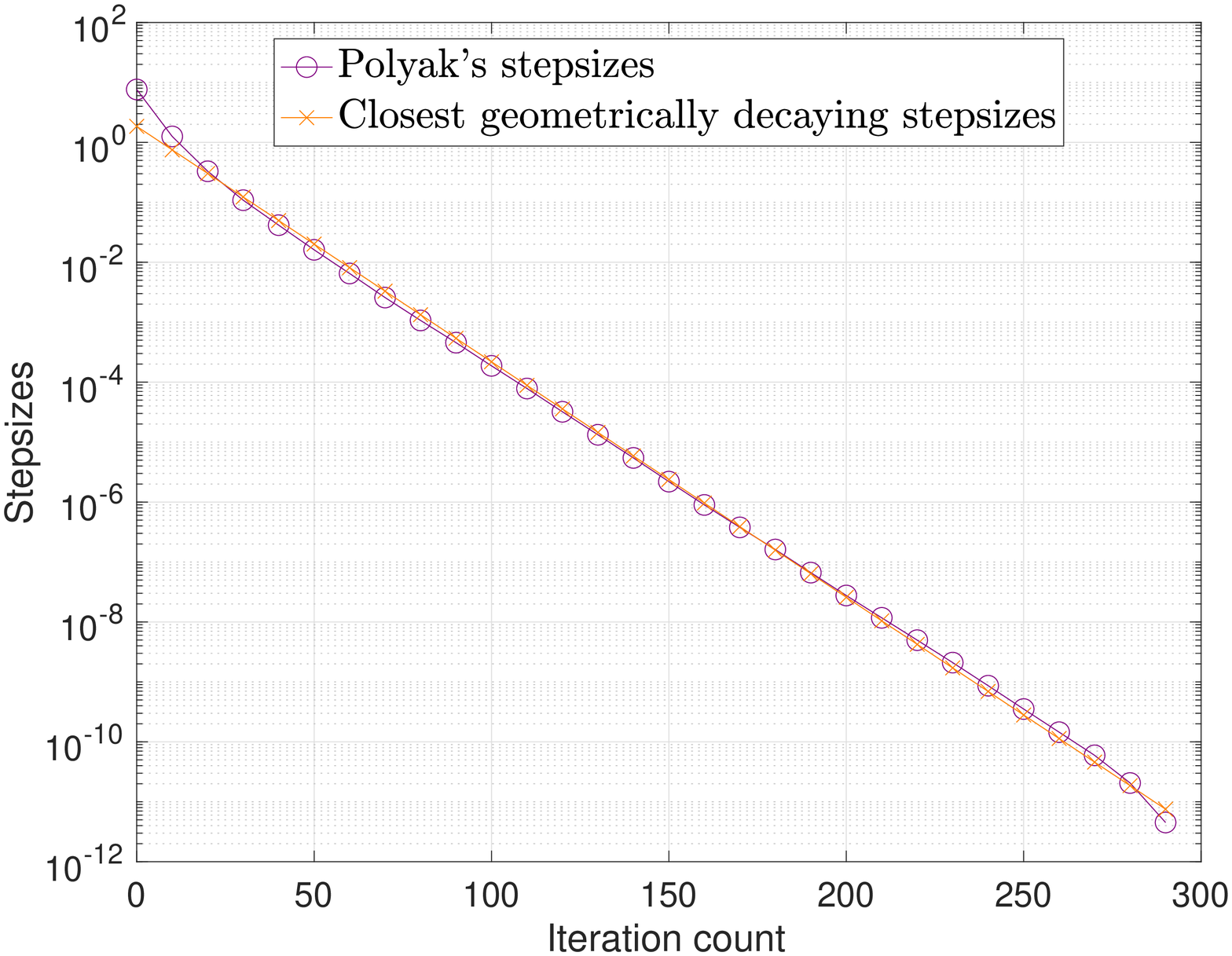} \\
{\small (c)} & {\small (d)}
\end{tabular}
\caption{Performance comparisons of \texttt{ScaledSM} for matrix sensing using  geometrically decaying stepsizes with parameters $(\lambda, q)$ and Polyak's stepsizes, where we fix $n=100$, $r=10$, $m=8nr$, $\kappa=10$, and $p_s=0.2$: (a) the final relative error for various combinations of $(\lambda, q)$, (b) the relative error versus iteration count for fixed $q=0.91$ and varying $\lambda$, (c) the relative error versus iteration count for fixed $\lambda=5$ and varying $q$, and (d) shows properly tuned geometrically decaying stepsizes with $\lambda=1.85$ and $q=0.91$ essentially match Polyak's stepsizes.}\label{fig:MS_stepsizes}
\end{figure*}

\begin{figure*}[!ht]
\centering
\begin{tabular}{cc}
\includegraphics[width=0.48\textwidth]{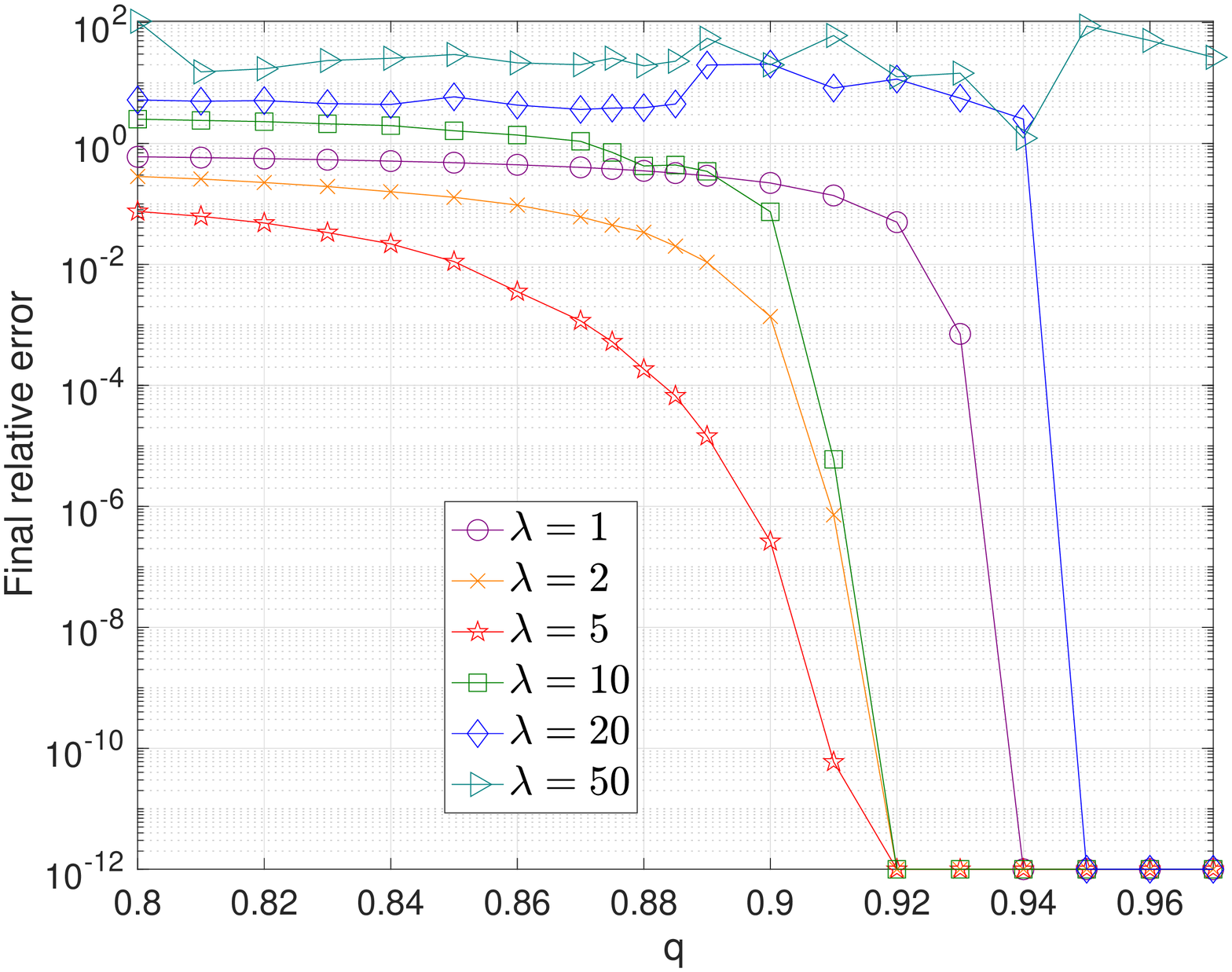} & 
\includegraphics[width=0.48\textwidth]{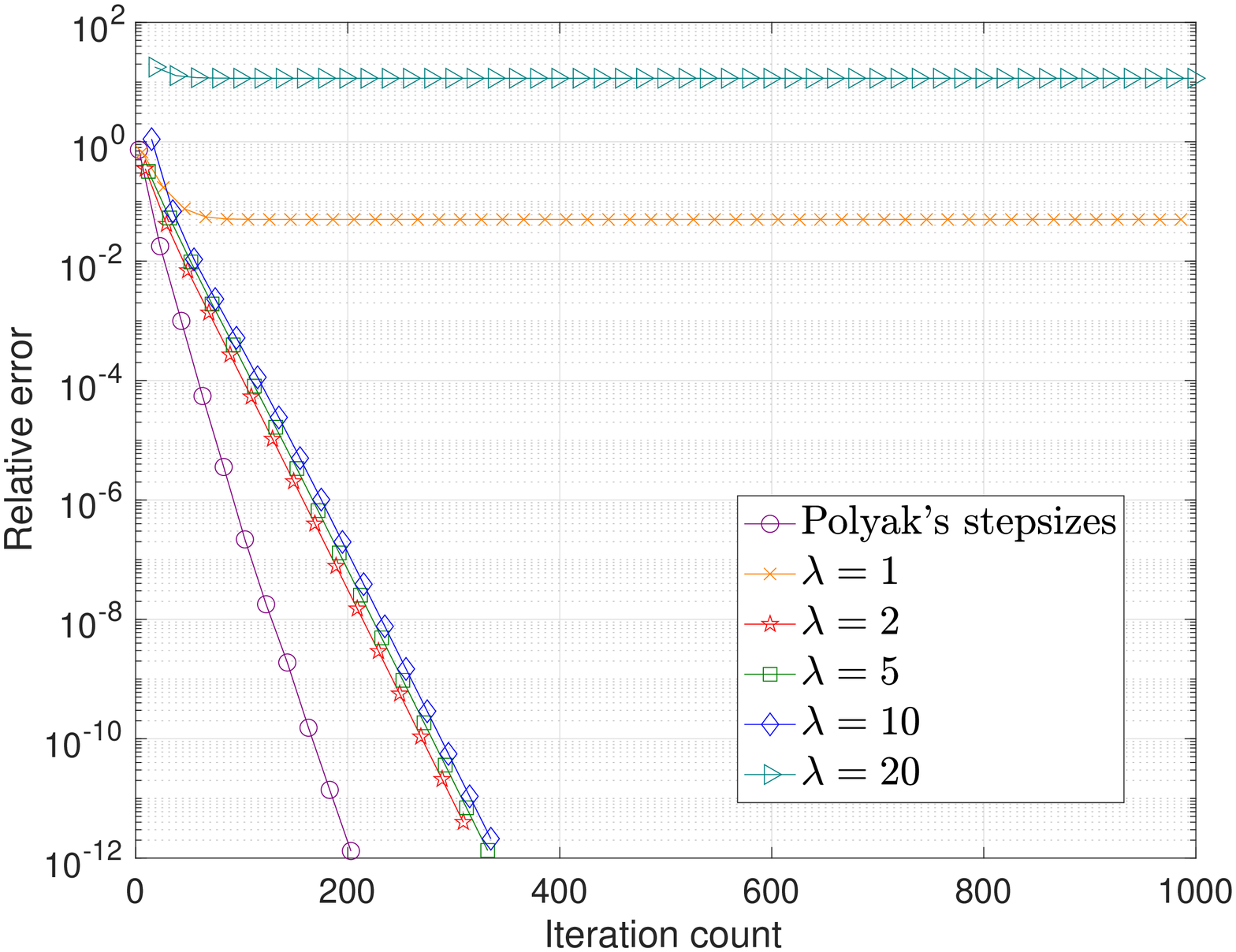} \\
{\small (a)} & {\small (b)} \\
\includegraphics[width=0.48\textwidth]{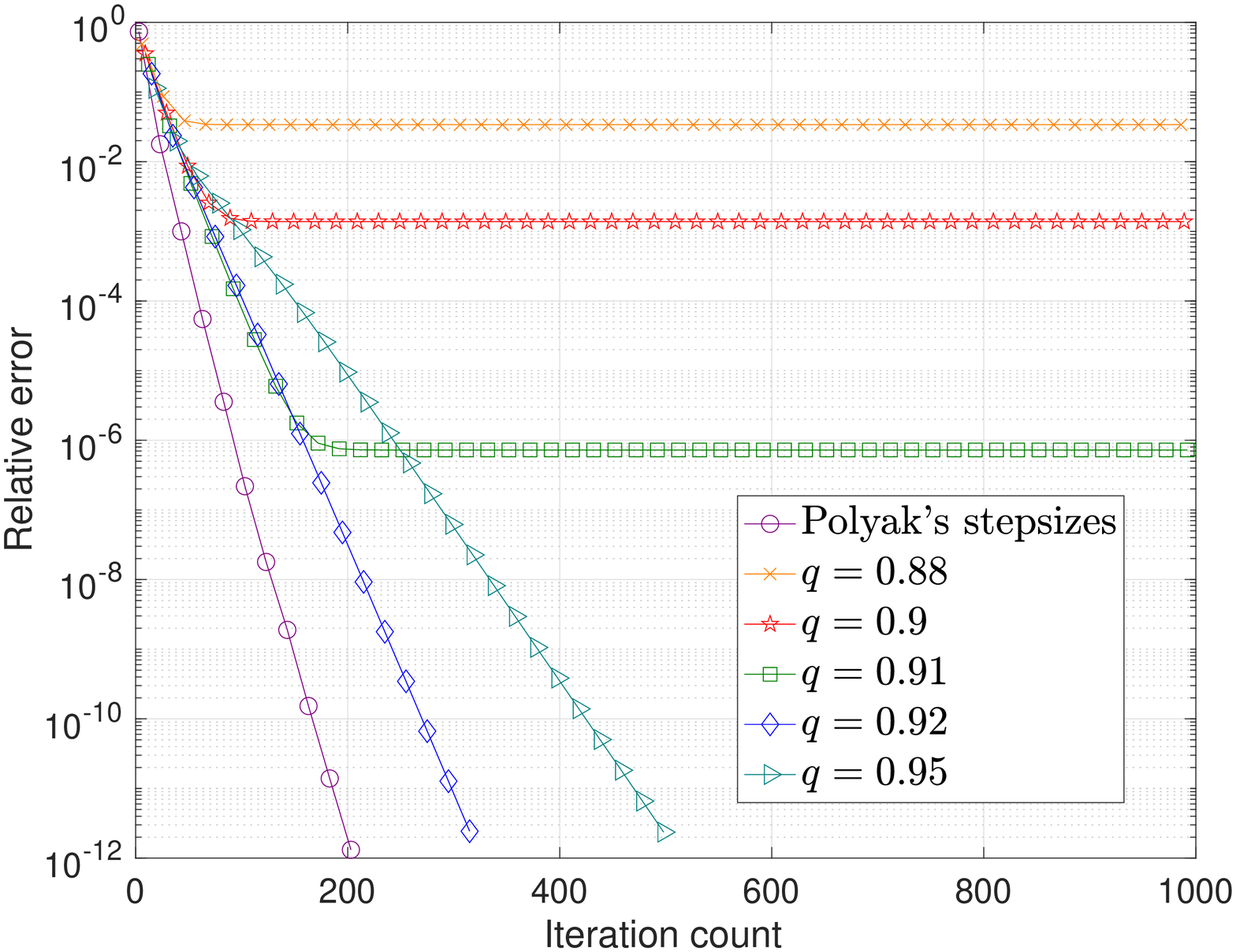} & 
\includegraphics[width=0.48\textwidth]{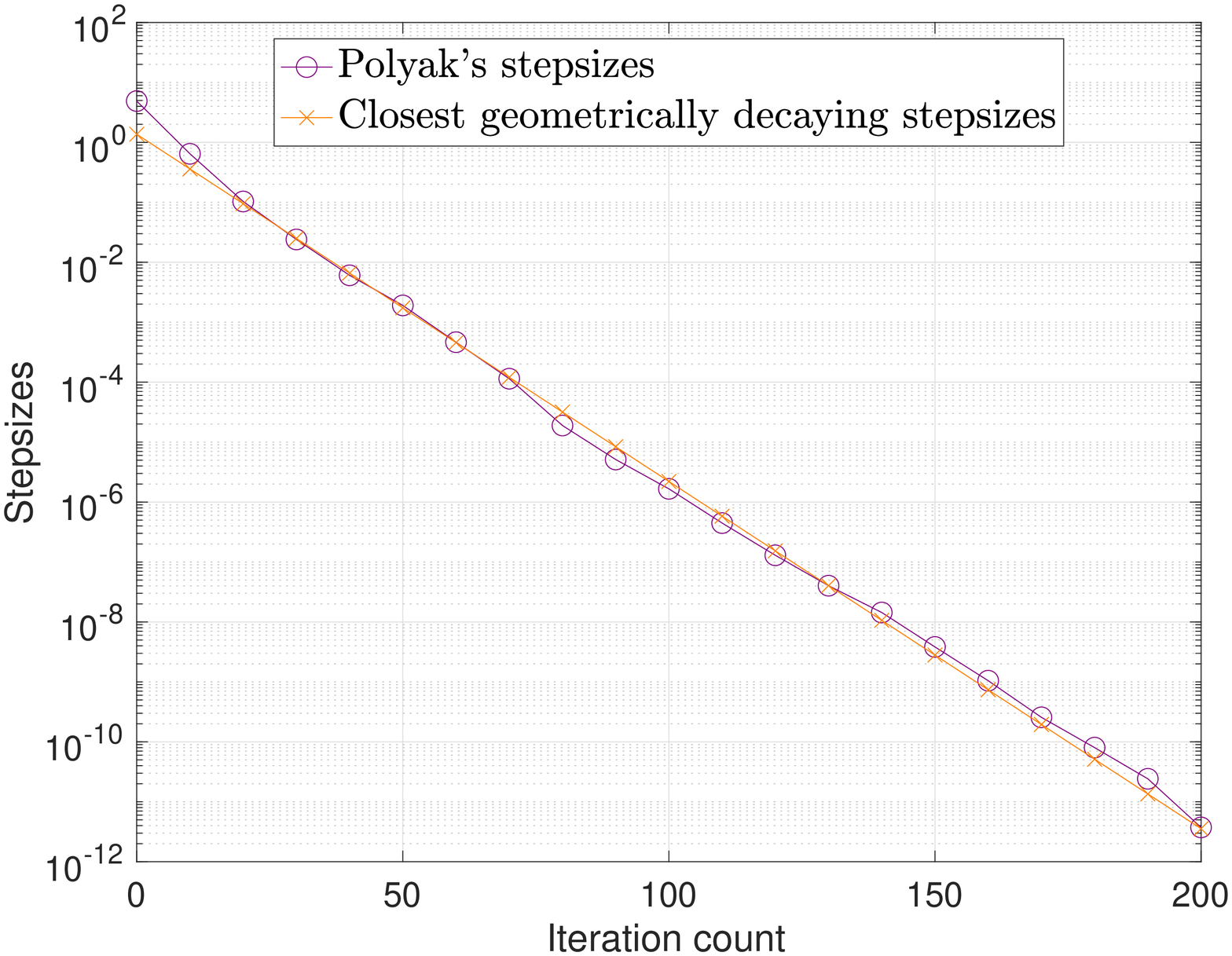} \\
{\small (c)} & {\small (d)}
\end{tabular}
\caption{Performance comparisons of \texttt{ScaledSM} for quadratic sampling using  geometrically decaying stepsizes with parameters $(\lambda, q)$ and Polyak's stepsizes, where we fix $n=100$, $r=5$, $m=8nr$, $\kappa=10$, and $p_s=0.2$: (a) the final relative error for various combinations of $(\lambda, q)$, (b) the relative error versus iteration count for fixed $q=0.92$ and varying $\lambda$, (c) the relative error versus iteration count for fixed $\lambda=2$ and varying $q$, and (d) shows properly tuned geometrically decaying stepsizes with $\lambda=1.36$ and $q=0.88$ essentially match Polyak's stepsizes.}\label{fig:QS_stepsizes} 
\end{figure*}

\section{Discussions} \label{sec:discussion}

This paper proposes scaled subgradient methods to minimize a family of nonsmooth and nonconvex formulations for low-rank matrix recovery---in particular, the residual sum of absolute errors---and guarantees its convergence at a rate that is almost dimension-free and independent of the condition number, even in the presence of corruptions. We illustrate the effectiveness of our approach by providing state-of-the-art performance guarantees for robust low-rank matrix sensing and quadratic sampling. In the future, it is of interest to study the performance of scaled subgradient methods for other signal estimation and statistical inference tasks, such as training student-teacher neural networks \cite{davis2020stochastic}, as well as using random initializations \cite{chen2019gradient}.

\section*{Acknowledgements}

The work of T.~Tong and Y.~Chi is supported in part by ONR under the grants N00014-18-1-2142 and N00014-19-1-2404, by ARO under the grant W911NF-18-1-0303, and by NSF under the grants CAREER ECCS-1818571, CCF-1806154 and CCF-1901199.

\bibliographystyle{alphaabbr}
\bibliography{bibfileNonconvex_TSP}

\newcommand{\etalchar}[1]{$^{#1}$}
\begin{thebibliography}{WGMM13}

\bibitem[BL20]{bahmani2020low}
S.~Bahmani and K.~Lee.
\newblock Low-rank matrix estimation from rank-one projections by unlifted
  convex optimization.
\newblock {\em arXiv preprint arXiv:2004.02718}, 2020.

\bibitem[CC18]{chen2018harnessing}
Y.~Chen and Y.~Chi.
\newblock Harnessing structures in big data via guaranteed low-rank matrix
  estimation: Recent theory and fast algorithms via convex and nonconvex
  optimization.
\newblock {\em IEEE Signal Processing Magazine}, 35(4):14 -- 31, 2018.

\bibitem[CCD{\etalchar{+}}21]{charisopoulos2019low}
V.~Charisopoulos, Y.~Chen, D.~Davis, M.~D{\'\i}az, L.~Ding, and
  D.~Drusvyatskiy.
\newblock Low-rank matrix recovery with composite optimization: good
  conditioning and rapid convergence.
\newblock {\em Foundations of Computational Mathematics}, pages 1--89, 2021.

\bibitem[CCFM19]{chen2019gradient}
Y.~Chen, Y.~Chi, J.~Fan, and C.~Ma.
\newblock Gradient descent with random initialization: Fast global convergence
  for nonconvex phase retrieval.
\newblock {\em Mathematical Programming}, 176(1-2):5--37, 2019.

\bibitem[CCG15]{chen2015exact}
Y.~Chen, Y.~Chi, and A.~Goldsmith.
\newblock Exact and stable covariance estimation from quadratic sampling via
  convex programming.
\newblock {\em IEEE Transactions on Information Theory}, 61(7):4034--4059, July
  2015.

\bibitem[CL16]{chi2016kaczmarz}
Y.~Chi and Y.~M. Lu.
\newblock {K}aczmarz method for solving quadratic equations.
\newblock {\em IEEE Signal Processing Letters}, 23(9):1183--1187, 2016.

\bibitem[CLC19]{chi2019nonconvex}
Y.~Chi, Y.~M. Lu, and Y.~Chen.
\newblock Nonconvex optimization meets low-rank matrix factorization: An
  overview.
\newblock {\em IEEE Transactions on Signal Processing}, 67(20):5239--5269,
  2019.

\bibitem[CLS15]{candes2015phase}
E.~Cand\`es, X.~Li, and M.~Soltanolkotabi.
\newblock Phase retrieval via {W}irtinger flow: Theory and algorithms.
\newblock {\em Information Theory, IEEE Transactions on}, 61(4):1985--2007,
  2015.

\bibitem[DDKL20]{davis2020stochastic}
D.~Davis, D.~Drusvyatskiy, S.~Kakade, and J.~D. Lee.
\newblock Stochastic subgradient method converges on tame functions.
\newblock {\em Foundations of computational mathematics}, 20(1):119--154, 2020.

\bibitem[DDP17]{davis2017nonsmooth}
D.~Davis, D.~Drusvyatskiy, and C.~Paquette.
\newblock The nonsmooth landscape of phase retrieval.
\newblock {\em arXiv preprint arXiv:1711.03247}, 2017.

\bibitem[D{\'\i}a19]{diaz2019nonsmooth}
M.~D{\'\i}az.
\newblock The nonsmooth landscape of blind deconvolution.
\newblock {\em arXiv preprint arXiv:1911.08526}, 2019.

\bibitem[DR16]{davenport2016overview}
M.~A. Davenport and J.~Romberg.
\newblock An overview of low-rank matrix recovery from incomplete observations.
\newblock {\em IEEE Journal of Selected Topics in Signal Processing},
  10(4):608--622, 2016.

\bibitem[DR19]{duchi2019solving}
J.~C. Duchi and F.~Ruan.
\newblock Solving (most) of a set of quadratic equalities: Composite
  optimization for robust phase retrieval.
\newblock {\em Information and Inference: A Journal of the IMA}, 8(3):471--529,
  2019.

\bibitem[Gof77]{goffin1977convergence}
J.-L. Goffin.
\newblock On convergence rates of subgradient optimization methods.
\newblock {\em Mathematical programming}, 13(1):329--347, 1977.

\bibitem[Han17]{hand2017phaselift}
P.~Hand.
\newblock Phaselift is robust to a constant fraction of arbitrary errors.
\newblock {\em Applied and Computational Harmonic Analysis}, 42(3):550--562,
  2017.

\bibitem[HW14]{hardt2014fast}
M.~Hardt and M.~Wootters.
\newblock Fast matrix completion without the condition number.
\newblock In {\em Proceedings of The 27th Conference on Learning Theory}, pages
  638--678, 2014.

\bibitem[JMD10]{jain2010guaranteed}
P.~Jain, R.~Meka, and I.~S. Dhillon.
\newblock Guaranteed rank minimization via singular value projection.
\newblock In {\em Advances in Neural Information Processing Systems}, pages
  937--945, 2010.

\bibitem[JNS13]{jain2013low}
P.~Jain, P.~Netrapalli, and S.~Sanghavi.
\newblock Low-rank matrix completion using alternating minimization.
\newblock In {\em Proceedings of the forty-fifth annual ACM symposium on Theory
  of computing}, pages 665--674. ACM, 2013.

\bibitem[LCZL20]{li2020non}
Y.~Li, Y.~Chi, H.~Zhang, and Y.~Liang.
\newblock Non-convex low-rank matrix recovery with arbitrary outliers via
  median-truncated gradient descent.
\newblock {\em Information and Inference: A Journal of the IMA}, 9(2):289--325,
  2020.

\bibitem[Li13]{li2013compressed}
X.~Li.
\newblock Compressed sensing and matrix completion with constant proportion of
  corruptions.
\newblock {\em Constructive Approximation}, 37(1):73--99, 2013.

\bibitem[LMCC21]{li2018nonconvex}
Y.~{Li}, C.~{Ma}, Y.~{Chen}, and Y.~{Chi}.
\newblock Nonconvex matrix factorization from rank-one measurements.
\newblock {\em IEEE Transactions on Information Theory}, 67(3):1928--1950,
  2021.

\bibitem[LSC17]{li2017low}
Y.~Li, Y.~Sun, and Y.~Chi.
\newblock Low-rank positive semidefinite matrix recovery from corrupted
  rank-one measurements.
\newblock {\em IEEE Transactions on Signal Processing}, 65(2):397--408, 2017.

\bibitem[LZSV20]{li2020nonconvex}
X.~Li, Z.~Zhu, A.~M.-C. So, and R.~Vidal.
\newblock Nonconvex robust low-rank matrix recovery.
\newblock {\em SIAM Journal on Optimization}, 30(1):660--686, 2020.

\bibitem[MAS12]{mishra2012riemannian}
B.~Mishra, K.~A. Apuroop, and R.~Sepulchre.
\newblock A {R}iemannian geometry for low-rank matrix completion.
\newblock {\em arXiv preprint arXiv:1211.1550}, 2012.

\bibitem[MLC21]{ma2021beyond}
C.~Ma, Y.~Li, and Y.~Chi.
\newblock Beyond {P}rocrustes: Balancing-free gradient descent for asymmetric
  low-rank matrix sensing.
\newblock {\em IEEE Transactions on Signal Processing}, 69:867--877, 2021.

\bibitem[MWCC19]{ma2017implicit}
C.~Ma, K.~Wang, Y.~Chi, and Y.~Chen.
\newblock Implicit regularization in nonconvex statistical estimation: Gradient
  descent converges linearly for phase retrieval, matrix completion, and blind
  deconvolution.
\newblock {\em Foundations of Computational Mathematics}, pages 1--182, 2019.

\bibitem[QZEW17]{qu2017convolutional}
Q.~Qing, Y.~Zhang, Y.~Eldar, and J.~Wright.
\newblock Convolutional phase retrieval via gradient descent.
\newblock {\em Neural Information Processing Systems}, 2017.

\bibitem[RFP10]{recht2010guaranteed}
B.~Recht, M.~Fazel, and P.~A. Parrilo.
\newblock Guaranteed minimum-rank solutions of linear matrix equations via
  nuclear norm minimization.
\newblock {\em SIAM review}, 52(3):471--501, 2010.

\bibitem[SWW17]{sanghavi2017local}
S.~Sanghavi, R.~Ward, and C.~D. White.
\newblock The local convexity of solving systems of quadratic equations.
\newblock {\em Results in Mathematics}, 71(3-4):569--608, 2017.

\bibitem[TBS{\etalchar{+}}16]{tu2015low}
S.~Tu, R.~Boczar, M.~Simchowitz, M.~Soltanolkotabi, and B.~Recht.
\newblock Low-rank solutions of linear matrix equations via {P}rocrustes flow.
\newblock In {\em International Conference Machine Learning}, pages 964--973,
  2016.

\bibitem[TMC20]{tong2020accelerating}
T.~Tong, C.~Ma, and Y.~Chi.
\newblock Accelerating ill-conditioned low-rank matrix estimation via scaled
  gradient descent.
\newblock {\em arXiv preprint arXiv:2005.08898}, 2020.

\bibitem[TW16]{tanner2016low}
J.~Tanner and K.~Wei.
\newblock Low rank matrix completion by alternating steepest descent methods.
\newblock {\em Applied and Computational Harmonic Analysis}, 40(2):417--429,
  2016.

\bibitem[WGMM13]{wright2013compressive}
J.~Wright, A.~Ganesh, K.~Min, and Y.~Ma.
\newblock Compressive principal component pursuit.
\newblock {\em Information and Inference: A Journal of the IMA}, 2(1):32--68,
  2013.

\bibitem[WS13]{wang2013exact}
L.~Wang and A.~Singer.
\newblock Exact and stable recovery of rotations for robust synchronization.
\newblock {\em Information and Inference: A Journal of the IMA}, 2(2):145--193,
  2013.

\bibitem[ZCL16]{zhang2016provable}
H.~Zhang, Y.~Chi, and Y.~Liang.
\newblock Provable non-convex phase retrieval with outliers: Median truncated
  {W}irtinger flow.
\newblock In {\em International conference on machine learning}, pages
  1022--1031, 2016.

\bibitem[ZZLC17]{zhang2017reshaped}
H.~Zhang, Y.~Zhou, Y.~Liang, and Y.~Chi.
\newblock A nonconvex approach for phase retrieval: Reshaped {W}irtinger flow
  and incremental algorithms.
\newblock {\em Journal of Machine Learning Research}, 18(141):1--35, 2017.

\end{thebibliography}

\appendix

\section{Technical Lemmas}

This section gathers several useful lemmas that will be used in the proof. Throughout the appendix, we use $\bX_{\star}$ to denote the ground truth, with its compact SVD as $\bX_{\star}=\bU_{\star}\bSigma_{\star}\bV_{\star}^{\top}$, and $\bF_{\star}=\begin{bmatrix}\bL_{\star}\\ \bR_{\star}\end{bmatrix}=\begin{bmatrix}\bU_{\star}\bSigma_{\star}^{1/2}\\ \bV_{\star}\bSigma_{\star}^{1/2}\end{bmatrix}$.
 For any factor matrix $\bF\coloneqq\begin{bmatrix}\bL\\ \bR\end{bmatrix}\in\RR^{(n_{1}+n_2)\times r}$, we define the optimal alignment matrix $\bQ$ between $\bF$ and $\bF_{\star}$ as 
\begin{align}
\bQ\coloneqq\argmin_{\bQ\in\GL(r)}\;\left\Vert (\bL\bQ-\bL_{\star})\bSigma_{\star}^{1/2}\right\Vert _{\fro}^{2}+\left\Vert (\bR\bQ^{-\top}-\bR_{\star})\bSigma_{\star}^{1/2}\right\Vert _{\fro}^{2},\label{eq:Q_def}
\end{align}
whenever the minimum is achieved.\footnote{If there exist multiple minimizers, we arbitrarily choose one as $\bQ$.}

\begin{lemma}[\cite{tong2020accelerating}]\label{lemma:Q_existence} Fix any factor matrix $\bF\coloneqq\begin{bmatrix}\bL\\ \bR\end{bmatrix}\in\RR^{(n_{1}+n_2)\times r}$. Suppose that
\begin{align*}
\dist(\bF,\bF_{\star})=\sqrt{\inf_{\bQ\in\GL(r)}\left\Vert \left(\bL\bQ-\bL_{\star}\right)\bSigma_{\star}^{1/2}\right\Vert _{\fro}^{2}+\left\Vert \left(\bR\bQ^{-\top}-\bR_{\star}\right)\bSigma_{\star}^{1/2}\right\Vert _{\fro}^{2}}<\sigma_{r}(\bX_{\star}),
\end{align*}
then the minimizer of the above minimization problem is attained at some $\bQ\in\GL(r)$, i.e.~the optimal alignment matrix $\bQ$ between $\bF$ and $\bF_{\star}$ exists. 
\end{lemma}

\begin{lemma}[\cite{tong2020accelerating}]\label{lemma:Procrustes} For any factor matrix $\bF\coloneqq \begin{bmatrix}\bL\\ \bR \end{bmatrix}\in\RR^{(n_{1}+n_{2})\times r}$, the following relation holds
\begin{align*}
\dist(\bF,\bF_{\star})\le\sqrt{\sqrt{2}+1}\,\|\bL\bR^{\top}-\bX_{\star}\|_{\fro}.
\end{align*}
\end{lemma}

\begin{lemma}[\cite{tong2020accelerating}]\label{lemma:Weyl} For any $\bL\in\RR^{n_{1}\times r},\bR\in\RR^{n_{2}\times r}$, denote $\bDelta_{L}\coloneqq\bL-\bL_{\star}$ and $\bDelta_{R}\coloneqq\bR-\bR_{\star}$. Suppose that $\max\{ \|\bDelta_{L}\bSigma_{\star}^{-1/2}\|_{\op}, \|\bDelta_{R}\bSigma_{\star}^{-1/2}\|_{\op} \}<1$, then one has 
\begin{subequations}
\begin{align}
\left\Vert \bL(\bL^{\top}\bL)^{-1}\bSigma_{\star}^{1/2}\right\Vert _{\op} & \le \frac{1}{1-\|\bDelta_{L}\bSigma_{\star}^{-1/2}\|_{\op}}; \label{eq:Weyl-1L} \\
\left\Vert \bR(\bR^{\top}\bR)^{-1}\bSigma_{\star}^{1/2}\right\Vert _{\op} & \le \frac{1}{1-\|\bDelta_{R}\bSigma_{\star}^{-1/2}\|_{\op}}; \label{eq:Weyl-1R} \\
\left\Vert \bL(\bL^{\top}\bL)^{-1}\bSigma_{\star}^{1/2}-\bU_{\star}\right\Vert _{\op}& \le\frac{\sqrt{2}\|\bDelta_{L}\bSigma_{\star}^{-1/2}\|_{\op}}{1-\|\bDelta_{L}\bSigma_{\star}^{-1/2}\|_{\op}}; \label{eq:Weyl-2L} \\
\left\Vert \bR(\bR^{\top}\bR)^{-1}\bSigma_{\star}^{1/2}-\bV_{\star}\right\Vert _{\op}& \le\frac{\sqrt{2}\|\bDelta_{R}\bSigma_{\star}^{-1/2}\|_{\op}}{1-\|\bDelta_{R}\bSigma_{\star}^{-1/2}\|_{\op}}. \label{eq:Weyl-2R}
\end{align}
\end{subequations}
\end{lemma}

\begin{lemma}[\cite{tong2020accelerating}]\label{lemma:norm_Fr_variation} Recall the partial Frobenius norm 
\begin{align}
\|\bX\|_{\fro,r}\coloneqq \sqrt{ \sum_{i=1}^r \sigma_{i}^2(\bX) } =\|\cP_{r}(\bX)\|_{\fro},\label{eq:norm_Fr_def}
\end{align}
where  $\cP_{r}(\bX)$ is the best rank-$r$ approximation of $\bX$.
For any $\bX\in\RR^{n_{1}\times n_{2}}$ and $\bR\in\RR^{n_{2}\times r}$, one has
\begin{align}
\|\bX \bR \|_{\fro} \le \|\bX\|_{\fro,r}\|\bR\|_{\op}.\label{eq:norm_Fr_variation-3}
\end{align}
In addition, for any $\bX,\bar{\bX}\in\RR^{n_{1}\times n_{2}}$ with $\rank(\bar{\bX})\le r$, one has
\begin{align}
|\langle\bX,\bar{\bX}\rangle| \le\|\bX\|_{\fro,r}\|\bar{\bX}\|_{\fro}.\label{eq:norm_Fr_variation-2}
\end{align}
\end{lemma}

\begin{lemma}\label{lemma:subgrad_L} Suppose that $f(\cdot): \RR^{n_1\times n_2}\mapsto \RR$ is convex and rank-$r$ restricted $L$-Lipschitz continuous (cf.~Definition~\ref{def:restricted_Lipschitz}). Then for any subgradient $\bS\in\partial f(\bX)$, one has $\|\bS\|_{\fro,r} \le L$. 
\end{lemma}
\begin{proof} Fix any subgradient $\bS\in\partial f(\bX)$. By the definition of a subgradient, for any $\tilde{\bX}\in\RR^{n_{1}\times n_{2}}$, one has
\begin{align*}
f(\tilde{\bX})\ge f(\bX)+\langle\bS,\tilde{\bX}-\bX\rangle.
\end{align*} 
In particular, taking $\tilde{\bX}=\bX+\cP_{r}(\bS)$ arrives at
\begin{align}
f(\bX+\cP_{r}(\bS))\ge f(\bX)+\langle\bS,\cP_{r}(\bS)\rangle=f(\bX)+\|\bS\|_{\fro,r}^{2},\label{eq:subgrad_bound}
\end{align}
where the last equality follows from the definition \eqref{eq:norm_Fr_def}. Note that $\cP_{r}(\bS)$ has rank at most $r$. By the rank-$r$ restricted $L$-Lipschitz continuity of $f(\cdot)$, we have
\begin{align*}
f(\bX+\cP_{r}(\bS))-f(\bX)\le L\|\cP_{r}(\bS)\|_{\fro}=L\|\bS\|_{\fro,r}.
\end{align*}
Combining the above inequality with \eqref{eq:subgrad_bound}, we conclude $\|\bS\|_{\fro,r} \le L$.
\end{proof}

\section{Proof of Theorem~\ref{thm:scaledSM}}\label{proof:theorem_scaledSM}

Suppose that the $t$-th iterate $\bF_{t}$ obeys the condition
\begin{align}
\dist(\bF_{t},\bF_{\star})\le 0.02\sigma_{r}(\bX_{\star})/\chi_{f}.\label{eq:induction_hypothesis}
\end{align} 
Lemma~\ref{lemma:Q_existence} ensures that $\bQ_{t}$, the optimal alignment matrix between $\bF_{t}$ and $\bF_{\star}$ exists. For notational convenience, we denote $\bL\coloneqq\bL_{t}\bQ_{t}$, $\bR\coloneqq\bR_{t}\bQ_{t}^{-\top}$, $\bDelta_{L}\coloneqq\bL-\bL_{\star}$, $\bDelta_{R}\coloneqq\bR-\bR_{\star}$, $\bS\coloneqq\bS_{t}$, and $\epsilon\coloneqq0.02/\chi_{f}$.
By the definition
\begin{align}
\dist(\bF_{t},\bF_{\star})=\sqrt{\|\bDelta_{L}\bSigma_{\star}^{1/2}\|_{\fro}^{2}+\|\bDelta_{R}\bSigma_{\star}^{1/2}\|_{\fro}^{2}}\label{eq:dist_Ft}
\end{align}
and the relation $\|\bA\bB\|_{\fro} \ge \|\bA\|_{\fro}\sigma_{r}(\bB) \ge \|\bA\|\sigma_{r}(\bB)$, we have
\begin{align}
\max\{\|\bDelta_{L}\bSigma_{\star}^{-1/2}\|_{\op}, \|\bDelta_{R}\bSigma_{\star}^{-1/2}\|_{\op}\}\le\epsilon.\label{eq:cond_SM}
\end{align}
We start by relating $ \|\bL\bR^{\top}-\bX_{\star}\|_{\fro}$ to $\dist(\bF_{t},\bF_{\star})$ given \eqref{eq:cond_SM}. Applying the triangle inequality to the basic relation $\bL\bR^{\top}-\bX_{\star} = \bL_{t}\bR_{t}^{\top}-\bX_{\star}=\bDelta_{L}\bR_{\star}^{\top}+\bL_{\star}\bDelta_{R}^{\top}+\bDelta_{L}\bDelta_{R}^{\top}$, we have
\begin{align}
\|\bL\bR^{\top}-\bX_{\star}\|_{\fro} & \le\|\bDelta_{L}\bR_{\star}^{\top}\|_{\fro}+\|\bL_{\star}\bDelta_{R}^{\top}\|_{\fro}+\|\bDelta_{L}\bDelta_{R}^{\top}\|_{\fro} \nonumber\\
 & \le \|\bDelta_{L}\bR_{\star}^{\top}\|_{\fro}+\|\bL_{\star}\bDelta_{R}^{\top}\|_{\fro}+\frac{1}{2}\|\bDelta_{L}\bSigma_{\star}^{-1/2}\|_{\op}\|\bDelta_{R}\bSigma_{\star}^{1/2}\|_{\fro}+\frac{1}{2}\|\bDelta_{L}\bSigma_{\star}^{1/2}\|_{\fro}\|\bDelta_{R}\bSigma_{\star}^{-1/2}\|_{\op} \nonumber\\
 & \le\left(1+\frac{1}{2}\max\{\|\bDelta_{L}\bSigma_{\star}^{-1/2}\|_{\op}, \|\bDelta_{R}\bSigma_{\star}^{-1/2}\|_{\op}\}\right)\left(\|\bDelta_{L}\bSigma_{\star}^{1/2}\|_{\fro}+\|\bDelta_{R}\bSigma_{\star}^{1/2}\|_{\fro}\right) \nonumber\\
 & \le \left(1+\frac{\epsilon}{2}\right)\sqrt{2}\sqrt{\|\bDelta_{L}\bSigma_{\star}^{1/2}\|_{\fro}^{2}+\|\bDelta_{R}\bSigma_{\star}^{1/2}\|_{\fro}^{2}}\le 1.5\dist(\bF_{t},\bF_{\star}),\label{eq:SM_dist_X}
\end{align}
where the last line uses the basic inequality $\|\bDelta_{L}\bSigma_{\star}^{1/2}\|_{\fro}+\|\bDelta_{R}\bSigma_{\star}^{1/2}\|_{\fro}\le\sqrt{2}\dist(\bF_{t},\bF_{\star})$ and \eqref{eq:cond_SM}.

From now on, we focus on proving the distance contraction. By the definition of $\dist(\bF_{t+1},\bF_{\star})$, one has
\begin{align}
\dist^{2}(\bF_{t+1},\bF_{\star}) & \le\left\Vert (\bL_{t+1}\bQ_{t}-\bL_{\star})\bSigma_{\star}^{1/2}\right\Vert _{\fro}^{2}+\left\Vert (\bR_{t+1}\bQ_{t}^{-\top}-\bR_{\star})\bSigma_{\star}^{1/2}\right\Vert _{\fro}^{2}.\label{eq:SM_expand}
\end{align}
We expand the first square in \eqref{eq:SM_expand} as
\begin{align}
\left\Vert(\bL_{t+1}\bQ_{t}-\bL_{\star})\bSigma_{\star}^{1/2}\right\Vert_{\fro}^2 & =\left\Vert \left(\bL-\eta_{t}\bS\bR(\bR^{\top}\bR)^{-1}-\bL_{\star}\right)\bSigma_{\star}^{1/2}\right\Vert_{\fro}^{2} \nonumber \\
 & =\|\bDelta_{L}\bSigma_{\star}^{1/2}\|_{\fro}^{2}-2\eta_{t}\left\langle\bS, \bDelta_{L}\bSigma_{\star}(\bR^{\top}\bR)^{-1}\bR^{\top}\right\rangle+\eta_{t}^{2}\left\Vert\bS\bR(\bR^{\top}\bR)^{-1}\bSigma_{\star}^{1/2}\right\Vert_{\fro}^{2} \nonumber \\
 & =\|\bDelta_{L}\bSigma_{\star}^{1/2}\|_{\fro}^{2}-2\eta_{t}\left\langle\bS, \bDelta_{L}\bR_{\star}^\top+\frac{1}{2}\bDelta_{L}\bDelta_{R}^{\top}\right\rangle+\eta_{t}^{2}\underbrace{\left\Vert\bS\bR(\bR^{\top}\bR)^{-1}\bSigma_{\star}^{1/2}\right\Vert_{\fro}^{2}}_{\mfk{S}_{1}} \nonumber \\
 & \quad -2\eta_{t}\underbrace{\left\langle\bS, \bDelta_{L}\bSigma_{\star}(\bR^{\top}\bR)^{-1}\bR^{\top}-\bDelta_{L}\bR_{\star}^\top-\frac{1}{2}\bDelta_{L}\bDelta_{R}^{\top}\right\rangle}_{\mfk{S}_{2}}, \label{eq:first_square}
\end{align}
where in the first line, we used the fact that the update rule \eqref{eq:scaledSM} is covariant with respect to $\bQ_{t}$, implying that
\begin{align*}
\bL_{t+1}\bQ_{t} = \bL - \eta_{t}\bS\bR(\bR^{\top}\bR)^{-1}.
\end{align*}

We proceed to bound $\mfk{S}_{1}$ and $\mfk{S}_{2}$. The term $\mfk{S}_{1}$ can be bounded by
\begin{align*}
\mfk{S}_{1} & \le\left\Vert\bS\bR(\bR^{\top}\bR)^{-1/2}\right\Vert_{\fro}^{2}\left\Vert(\bR^{\top}\bR)^{-1/2}\bSigma_{\star}^{1/2}\right\Vert_{\op}^{2} \\
 & \le\left\Vert\bS\bR(\bR^{\top}\bR)^{-1/2}\right\Vert_{\fro}^{2}\frac{1}{(1-\epsilon)^{2}},
\end{align*}
where the second line follows from the condition \eqref{eq:cond_SM} and Lemma~\ref{lemma:Weyl} (cf.~\eqref{eq:Weyl-1R}):
\begin{align*}
\left\Vert(\bR^{\top}\bR)^{-1/2}\bSigma_{\star}^{1/2}\right\Vert_{\op}=\left\Vert\bR(\bR^{\top}\bR)^{-1}\bSigma_{\star}^{1/2}\right\Vert_{\op} & \le\frac{1}{1-\epsilon}.
\end{align*}
For the term $\mfk{S}_{2}$, note that 
\begin{align*}
\bDelta_{L}\bSigma_{\star}(\bR^{\top}\bR)^{-1}\bR^{\top}-\bDelta_{L}\bR_{\star}^{\top}-\frac{1}{2}\bDelta_{L}\bDelta_{R}^{\top} = \bDelta_{L}\bSigma_{\star}^{1/2} \left(\bR(\bR^{\top}\bR)^{-1}\bSigma_{\star}^{1/2}-\bV_{\star} - \frac{1}{2}\bDelta_{R}\bSigma_{\star}^{-1/2} \right)^{\top} 
\end{align*} 
has rank at most $r$. Hence we can invoke Lemma~\ref{lemma:norm_Fr_variation} (cf.~\eqref{eq:norm_Fr_variation-2}) to obtain
\begin{align*}
|\mfk{S}_{2}| & \le\|\bS\|_{\fro,r}\left\Vert\bDelta_{L}\bSigma_{\star}(\bR^{\top}\bR)^{-1}\bR^{\top}-\bDelta_{L}\bR_{\star}^{\top}-\frac{1}{2}\bDelta_{L}\bDelta_{R}^{\top}\right\Vert_{\fro} \\
 & \le \|\bS\|_{\fro,r}\|\bDelta_{L}\bSigma_{\star}^{1/2}\|_{\fro}\left(\left\Vert\bR(\bR^{\top}\bR)^{-1}\bSigma_{\star}^{1/2}-\bV_{\star}\right\Vert_{\op}+\frac{1}{2}\|\bDelta_{R}\bSigma_{\star}^{-1/2}\|_{\op}\right) \\
 & \le L\left(\frac{\sqrt{2}\epsilon}{1-\epsilon}+\frac{\epsilon}{2}\right)\|\bDelta_{L}\bSigma_{\star}^{1/2}\|_{\fro},
\end{align*}
where the second line follows from the triangle inequality, and the third line follows from $\|\bS\|_{\fro,r} \le L$ (cf.~Lemma~\ref{lemma:subgrad_L}), \eqref{eq:cond_SM}, and 
\begin{align*}
\left\Vert\bR(\bR^{\top}\bR)^{-1}\bSigma_{\star}^{1/2}-\bV_{\star}\right\Vert_{\op} & \le\frac{\sqrt{2}\epsilon}{1-\epsilon}
\end{align*}
from Lemma~\ref{lemma:Weyl} (cf.~\eqref{eq:Weyl-2R}).

Plugging collectively the bounds for $\mfk{S}_{1}$ and $\mfk{S}_{2}$ into \eqref{eq:first_square} yields 
\begin{align*}
\left\Vert (\bL_{t+1}\bQ_{t}-\bL_{\star})\bSigma_{\star}^{1/2}\right\Vert _{\fro}^{2} & \le \|\bDelta_{L}\bSigma_{\star}^{1/2}\|_{\fro}^{2}-2\eta_{t}\left\langle\bS, \bDelta_{L}\bR_{\star}^\top+\frac{1}{2}\bDelta_{L}\bDelta_{R}^{\top}\right\rangle+\frac{\eta_{t}^{2}}{(1-\epsilon)^{2}}\left\Vert\bS\bR(\bR^{\top}\bR)^{-1/2}\right\Vert_{\fro}^{2} \\
 & \quad +\eta_{t}L\left(\frac{2\sqrt{2}\epsilon}{1-\epsilon}+\epsilon \right)\|\bDelta_{L}\bSigma_{\star}^{1/2}\|_{\fro}. 
\end{align*}
Similarly, we can obtain the control of $\|(\bR_{t+1}\bQ_{t}^{-\top}-\bR_{\star})\bSigma_{\star}^{1/2}\|_{\fro}^{2}$. Combine them together to reach
\begin{align*}
 & \dist^{2}(\bF_{t+1},\bF_{\star}) \le \|\bDelta_{L}\bSigma_{\star}^{1/2}\|_{\fro}^{2}+\|\bDelta_{R}\bSigma_{\star}^{1/2}\|_{\fro}^{2}-2\eta_{t}\left\langle\bS, \bDelta_{L}\bR_{\star}^{\top}+\bL_{\star}\bDelta_{R}^{\top}+\bDelta_{L}\bDelta_{R}^{\top}\right\rangle\\
 & \quad +\frac{\eta_{t}^{2}}{(1-\epsilon)^{2}}\left(\left\Vert\bS\bR(\bR^{\top}\bR)^{-1/2}\right\Vert_{\fro}^{2}+\left\Vert\bS^{\top}\bL(\bL^{\top}\bL)^{-1/2}\right\Vert_{\fro}^{2}\right)+\eta_{t}L\left(\frac{2\sqrt{2}\epsilon}{1-\epsilon}+\epsilon\right)\left(\|\bDelta_{L}\bSigma_{\star}^{1/2}\|_{\fro}+\|\bDelta_{R}\bSigma_{\star}^{1/2}\|_{\fro}\right).
\end{align*}
Using the subgradient optimality of $\bS$, we obtain
\begin{align*}
\langle\bS,\bDelta_{L}\bR_{\star}^{\top}+\bL_{\star}\bDelta_{R}^{\top}+\bDelta_{L}\bDelta_{R}^{\top}\rangle=\langle\bS,\bL\bR^{\top}-\bX_{\star}\rangle \ge f(\bL\bR^{\top})-f(\bX_{\star}),
\end{align*}
together with \eqref{eq:dist_Ft}, which further implies that 
\begin{multline}
 \dist^{2}(\bF_{t+1},\bF_{\star}) \le \dist^{2}(\bF_{t},\bF_{\star}) -2\eta_{t}\left(f(\bL\bR^{\top})-f(\bX_{\star})\right) \\
+\frac{\eta_{t}^{2}}{(1-\epsilon)^{2}}\left(\left\Vert\bS\bR(\bR^{\top}\bR)^{-1/2}\right\Vert_{\fro}^{2}+\left\Vert\bS^{\top}\bL(\bL^{\top}\bL)^{-1/2}\right\Vert_{\fro}^{2}\right)+\eta_{t}L\left(\frac{4\epsilon}{1-\epsilon}+\sqrt{2}\epsilon\right)\dist(\bF_{t},\bF_{\star}),
\label{eq:SM_bound}
\end{multline}
where the last term uses the basic inequality $\|\bDelta_{L}\bSigma_{\star}^{1/2}\|_{\fro}+\|\bDelta_{R}\bSigma_{\star}^{1/2}\|_{\fro}\le\sqrt{2}\dist(\bF_{t},\bF_{\star})$.

Before proceeding to different cases of stepsize schedules, we record two useful properties. First, by the restricted $\mu$-sharpness of $f(\cdot)$ together with Lemma~\ref{lemma:Procrustes}, we have 
\begin{align}
f(\bL\bR^{\top})-f(\bX_{\star})\ge\mu\|\bL\bR^{\top}-\bX_{\star}\|_{\fro}\ge\mu\sqrt{\sqrt{2}-1}\dist(\bF_{t},\bF_{\star}).\label{eq:SM_sharpness}
\end{align}
On the other end, by Lemma~\ref{lemma:norm_Fr_variation} (cf.~\eqref{eq:norm_Fr_variation-3}), we have
\begin{align}
\|\bS\bR(\bR^{\top}\bR)^{-1/2}\|_{\fro}^{2}+\|\bS^{\top}\bL(\bL^{\top}\bL)^{-1/2}\|_{\fro}^{2} & \le \|\bS\|_{\fro,r}^{2}\left(\|\bR(\bR^{\top}\bR)^{-1/2}\|_{\op}^{2}+\|\bL(\bL^{\top}\bL)^{-1/2}\|_{\op}^{2}\right) \nonumber \\
& \le 2L^{2},\label{eq:SM_subgrad_L}
\end{align}
where the second line follows from $\|\bS\|_{\fro,r} \le L$ (cf.~Lemma~\ref{lemma:subgrad_L}) and
\begin{align*}
\|\bR(\bR^{\top}\bR)^{-1/2}\|_{\op}^{2} =\|\bR(\bR^{\top}\bR)^{-1}\bR^{\top}\|_{\op} =1,\quad  \|\bL(\bL^{\top}\bL)^{-1/2}\|_{\op}^{2} = \|\bL(\bL^{\top}\bL)^{-1}\bL^{\top}\|_{\op} =1.
\end{align*}

\subsection{Convergence with Polyak's stepsizes}

Let $\eta_t = \eta_t^{\mathsf{P}}$ be the Polyak's stepsize in \eqref{eq:Polyak_stepsize}, which is
\begin{align}
\eta_{t} & =\frac{f(\bL_{t}\bR_{t}^{\top})-f(\bX_{\star})}{\|\bS_{t}\bR_{t}(\bR_{t}^{\top}\bR_{t})^{-1/2}\|_{\fro}^{2}+\|\bS_{t}^{\top}\bL_{t}(\bL_{t}^{\top}\bL_{t})^{-1/2}\|_{\fro}^{2}} \nonumber\\
 & = \frac{f(\bL\bR^{\top})-f(\bX_{\star})}{\|\bS\bR(\bR^{\top}\bR)^{-1/2}\|_{\fro}^{2}+\|\bS^{\top}\bL(\bL^{\top}\bL)^{-1/2}\|_{\fro}^{2}},\label{eq:polyak_stepsize_ref}
\end{align}
where the second line follows since $\bL_t\bR_t^{\top}=\bL\bR^{\top}$, $\bL_{t}(\bL_{t}^{\top}\bL_{t})^{-1}\bL_{t}^{\top} = \bL(\bL^{\top}\bL)^{-1}\bL^{\top}$ and $\bR_{t}(\bR_{t}^{\top}\bR_{t})^{-1}\bR_{t}^{\top} = \bR(\bR^{\top}\bR)^{-1}\bR^{\top}$. Plugging \eqref{eq:polyak_stepsize_ref} into \eqref{eq:SM_bound}, we have
\begin{align}
 \dist^{2}(\bF_{t+1},\bF_{\star}) & \le \dist^{2}(\bF_{t},\bF_{\star})-\eta_{t}\left(2-\frac{1}{(1-\epsilon)^{2}}\right)\left(f(\bL\bR^{\top})-f(\bX_{\star})\right)+\eta_{t}L\left(\frac{4\epsilon}{1-\epsilon}+\sqrt{2}\epsilon \right)\dist(\bF_{t},\bF_{\star}) \nonumber \\
 & \le\dist^{2}(\bF_{t},\bF_{\star})-\eta_{t}\mu\left(\sqrt{\sqrt{2}-1}\left(2-\frac{1}{(1-\epsilon)^{2}}\right)-\chi_{f}\left(\frac{4\epsilon}{1-\epsilon}+\sqrt{2}\epsilon\right)\right)\dist(\bF_{t},\bF_{\star}), \label{eq:dist_polyak_intermediate}
\end{align}
where the second line follows from \eqref{eq:SM_sharpness} and $\chi_f = L/\mu$. 

To continue, combining \eqref{eq:SM_sharpness} and \eqref{eq:SM_subgrad_L}, we can lower bound the Polyak's stepsize \eqref{eq:polyak_stepsize_ref} as
\begin{align*}
\eta_{t} & \ge\frac{\sqrt{\sqrt{2}-1}\mu\dist(\bF_{t},\bF_{\star})}{2L^{2}}.
\end{align*}
This, combined with \eqref{eq:dist_polyak_intermediate}, leads to 
\begin{align*}
\dist^{2}(\bF_{t+1},\bF_{\star}) \le \rho(\epsilon,\chi_{f}) \dist^{2}(\bF_{t},\bF_{\star}),
\end{align*}
where the contraction rate $\rho(\epsilon,\chi_{f})$ is
\begin{align}
\rho(\epsilon,\chi_{f})\coloneqq 1-\frac{\sqrt{\sqrt{2}-1}}{2\chi_{f}^{2}}\left(\sqrt{\sqrt{2}-1}\left(2-\frac{1}{(1-\epsilon)^{2}}\right)-\chi_{f}\left(\frac{4\epsilon}{1-\epsilon}+\sqrt{2}\epsilon\right)\right).\label{eq:contraction_rate}
\end{align}
Under the condition $\epsilon=0.02/\chi_{f}$, we calculate $(1-\rho(\epsilon,\chi_{f}))\chi_f^2 $ as 
\begin{align*}
 & \frac{\sqrt{\sqrt{2}-1}}{2}\left(\sqrt{\sqrt{2}-1}\left(2-\frac{1}{(1-\epsilon)^{2}}\right)-\chi_{f}\left(\frac{4\epsilon}{1-\epsilon}+\sqrt{2}\epsilon\right)\right)  \\
 & \quad \ge 0.32\left(0.64\times\left(2-\frac{1}{0.98^{2}}\right)-0.02\left(\frac{4}{0.98}+\sqrt{2}\right)\right) \ge 0.16,
\end{align*}
thus $\rho(\epsilon,\chi_{f}) \le 1-0.16/\chi_{f}^{2}$.  We conclude that
\begin{align*}
\dist^{2}(\bF_{t+1},\bF_{\star}) \le (1-0.16/\chi_{f}^{2})\dist^{2}(\bF_{t},\bF_{\star}),
\end{align*}
which is the desired claim.

\subsection{Convergence with geometrically decaying stepsizes}

Let $\eta_t = \eta_t^{\mathsf{G}}$ be the geometrically decaying stepsize in \eqref{eq:geometric_stepsize}, which is
\begin{align*}
\eta_{t}=\frac{\lambda q^{t}}{\sqrt{\left\Vert\bS\bR(\bR^{\top}\bR)^{-1/2}\right\Vert_{\fro}^{2}+\left\Vert\bS^{\top}\bL(\bL^{\top}\bL)^{-1/2}\right\Vert_{\fro}^{2}}}.
\end{align*}
Plugging the above into \eqref{eq:SM_bound}, we have
\begin{align*}
\dist^{2}(\bF_{t+1},\bF_{\star}) & \le \dist^{2}(\bF_{t},\bF_{\star})-\eta_{t}\mu\left(2\sqrt{\sqrt{2}-1}-\chi_{f}\left(\frac{4\epsilon}{1-\epsilon}+\sqrt{2}\epsilon\right)\right)\dist(\bF_{t},\bF_{\star})+\frac{\lambda^{2}q^{2t}}{(1-\epsilon)^{2}} \\
 & \le \dist^{2}(\bF_{t},\bF_{\star})-\frac{\lambda q^{t}}{\sqrt{2}\chi_{f}}\left(2\sqrt{\sqrt{2}-1}-\chi_{f}\left(\frac{4\epsilon}{1-\epsilon}+\sqrt{2}\epsilon\right)\right)\dist(\bF_{t},\bF_{\star})+\frac{\lambda^{2}q^{2t}}{(1-\epsilon)^{2}},
\end{align*}
where the first line follows from  \eqref{eq:SM_sharpness} and $\chi_f = L/\mu$, and the second line follows from $\eta_t \ge \frac{\lambda q^t}{\sqrt{2} L}$
due to \eqref{eq:SM_subgrad_L}. We now aim to show that
\begin{align*}
\dist(\bF_{t},\bF_{\star})\le(1-0.16/\chi_{f}^{2})^{t/2}0.02\sigma_{r}(\bX_{\star})/\chi_{f}
\end{align*}
in an inductive manner. Assume the above induction hypothesis holds at the $t$-iteration. By the setting of parameters, i.e.
\begin{align*}
\lambda q^{t}=\sqrt{\frac{\sqrt{2}-1}{2}}(1-0.16/\chi_{f}^{2})^{t/2}0.02\sigma_{r}(\bX_{\star})/\chi_{f}^{2},
\end{align*}
we have
\begin{align*}
 \dist^{2}(\bF_{t+1},\bF_{\star}) \le \rho(\epsilon,\chi_{f}) (1-0.16/\chi_{f}^{2})^{t}(0.02\sigma_{r}(\bX_{\star})/\chi_{f})^{2},
\end{align*}
where the contraction rate $\rho(\epsilon,\chi_{f})$ matches exactly \eqref{eq:contraction_rate}. Therefore, under the condition $\epsilon=0.02/\chi_{f}$, we have $\rho(\epsilon,\chi_{f}) \le 1-0.16/\chi_{f}^{2}$, thus we conclude that
\begin{align*}
\dist(\bF_{t+1},\bF_{\star}) \le(1-0.16/\chi_{f}^{2})^{\frac{t+1}{2}}0.02\sigma_{r}(\bX_{\star})/\chi_{f},
\end{align*}
which is the desired claim.

\section{Proof of Theorem~\ref{thm:scaledSM_noisy}}\label{proof:theorem_scaledSM_noisy} 

We start by introducing the short-hand notation $d_{t}\coloneqq(1-0.13/\chi_{f}^{2})^{t/2}0.02\sigma_{r}(\bX_{\star})/\chi_{f}$. The parameters are set as
\begin{align*}
\lambda q^{t}=\sqrt{\frac{\sqrt{2}-1}{2}}(1-0.13/\chi_{f}^{2})^{t/2} 0.02\sigma_{r}(\bX_{\star})/\chi_{f}^{2}=\sqrt{\frac{\sqrt{2}-1}{2}}\frac{d_{t}}{\chi_{f}}.
\end{align*}
Therefore, the geometric stepsize 
\begin{align*}
\eta_{t}=\frac{\lambda q^{t}}{\sqrt{\left\Vert\bS\bR(\bR^{\top}\bR)^{-1/2}\right\Vert_{\fro}^{2}+\left\Vert\bS^{\top}\bL(\bL^{\top}\bL)^{-1/2}\right\Vert_{\fro}^{2}}},
\end{align*}
in view of \eqref{eq:SM_subgrad_L}, satisfies
\begin{align}
\eta_t \ge \frac{\lambda q^t}{\sqrt{2} L} = \frac{\sqrt{\sqrt{2}-1}}{2}\frac{d_{t}}{\chi_{f}^2\mu}. \label{eq:stepsize_lower_bound}
\end{align}

Follow the same derivations as the proof of Theorem~\ref{thm:scaledSM} until \eqref{eq:SM_bound}. Plugging the stepsize \eqref{eq:stepsize_lower_bound}
into \eqref{eq:SM_bound}, together with the approximate restricted sharpness property
\begin{align*}
f(\bL\bR^{\top})-f(\bX_{\star}) & \ge\mu\|\bL\bR^{\top}-\bX_{\star}\|_{\fro}-\xi \ge\sqrt{\sqrt{2}-1}\mu\dist(\bF_{t},\bF_{\star})-\xi,
\end{align*}
we have
\begin{align*}
\dist^{2}(\bF_{t+1},\bF_{\star}) & \le \dist^{2}(\bF_{t},\bF_{\star})-\eta_{t}\mu\left(2\sqrt{\sqrt{2}-1}-\left(\frac{4}{1-\epsilon}+\sqrt{2}\right)\epsilon\chi_{f}\right)\dist(\bF_{t},\bF_{\star})+\frac{\lambda^{2}q^{2t}}{(1-\epsilon)^{2}} + 2\eta_{t}\xi.
\end{align*}
Under the conditions $\chi_f\ge 1$ and $\epsilon=0.02/\chi_{f}\le 0.02$, the above relation can be simplified to
\begin{align}
\dist^{2}(\bF_{t+1},\bF_{\star}) & \le \dist^{2}(\bF_{t},\bF_{\star})-1.177\eta_{t}\mu \dist(\bF_{t},\bF_{\star})+ \frac{0.216}{\chi_f^2}  d_t^2 + 2\eta_{t}\xi.\label{eq:dist_induction}
\end{align}
We next prove the theorem by induction, where the base case is established trivially by the initial condition. By the induction hypothesis, the distance at the $t$-th iterate is bounded by
\begin{align*}
\dist(\bF_{t},\bF_{\star})\le \max\left\{d_{t},20\xi/\mu\right\}.
\end{align*}
To obtain the control of $\dist(\bF_{t+1},\bF_{\star})$, we split the discussion in two cases.
\begin{enumerate}
\item If $d_{t}\ge 20\xi/\mu$, or equivalently, $\xi\le0.05\mu d_{t}$, in view of \eqref{eq:dist_induction}, we have
\begin{align*}
\dist^{2}(\bF_{t+1},\bF_{\star}) & \overset{\mathrm{(i)}}{\le} d_{t}^{2} - 1.177\eta_{t}\mu d_{t} + \frac{0.216}{\chi_{f}^{2}} d_t^2+ 0.1\eta_{t}\mu d_{t} \\
\quad & = d_{t}^{2} -1.077 \eta_{t}\mu d_{t} +\frac{0.216}{\chi_{f}^{2}} d_t^2 \\
\quad & \overset{\mathrm{(ii)}}{\le} d_{t}^{2}- \frac{0.346}{\chi_{f}^2}d_{t}^{2}+\frac{0.216}{\chi_{f}^{2}} d_t^2 \\
\quad & = (1-0.13/\chi_f^2)d_{t}^{2},
\end{align*}
where $\mathrm{(i)}$ uses $\xi\le0.05\mu d_{t}$, and $\mathrm{(ii)}$ uses the condition \eqref{eq:stepsize_lower_bound}. We conclude that $\dist(\bF_{t+1},\bF_{\star})\le(1-0.13/\chi_{f}^{2})^{1/2}d_{t}$. 

\item If $0\le d_{t}<20\xi/\mu$, we have
\begin{align*}
\dist^{2}(\bF_{t+1},\bF_{\star}) & \le \left(\frac{20\xi}{\mu}\right)^{2}-1.177 \eta_{t}\mu \frac{20\xi}{\mu} +\frac{0.216}{\chi_f^2}  d_t^2  + 2\eta_{t}\xi \\
 & = \left(\frac{20\xi}{\mu}\right)^{2}-1.077 \eta_{t}\mu\frac{20\xi}{\mu}+\frac{0.216}{\chi_f^2}  d_t^2   \\
 & \le \left(\frac{20\xi}{\mu}\right)^{2}-\frac{1.077\sqrt{\sqrt{2}-1}}{2\chi_{f}^{2}} d_{t}\frac{20\xi}{\mu} + \frac{0.216}{\chi_f^2}  d_t^2 \\
 & \le  \left(\frac{20\xi}{\mu}\right)^{2} - 0.13 d_t\frac{20\xi}{\mu} \\
 & \le \left(\frac{20\xi}{\mu}\right)^{2},
\end{align*}
where the third line uses the condition \eqref{eq:stepsize_lower_bound}, and the last line holds since $d_t\ge 0$.
\end{enumerate}
In sum, we conclude 
\begin{align*}
\dist(\bF_{t+1},\bF_{\star}) \le \max\left\{(1-0.13/\chi_{f}^{2})^{\frac{t+1}{2}}0.02\sigma_{r}(\bX_{\star})/\chi_{f}, 20\xi/\mu\right\},
\end{align*}
which is the desired claim.

\section{Proof of Proposition~\ref{prop:mixedRIP_clean}} \label{proof:prop_mixedRIP_clean}
For $\bX_1$ and $\bX_2$ where $\bX_1 - \bX_2$ has rank at most $2r$, we have
\begin{align*}
|f(\bX_{1})-f(\bX_{2})| &= \Big| \| \cA(\bX_{1}-\bX_{\star})\|_1 -\| \cA(\bX_{2}-\bX_{\star})\|_1 \Big| \\
& \le \|\cA(\bX_{1}-\bX_{2})\|_1 \le \delta_{2}\|\bX_{1}-\bX_{2}\|_{\fro},
\end{align*}
where the second line follows from the inverse triangle inequality and the assumed rank-$2r$ mixed-norm RIP (cf.~Definition~\ref{def:mixed_rip}) of $\cA(\cdot)$. As a result, we have $L=\delta_2$. On the other end,
we note
\begin{align*}
f(\bX)-f(\bX_{\star}) = \| \cA(\bX-\bX_{\star})\|_1 \ge \delta_{1}\|\bX-\bX_{\star}\|_{\fro},
\end{align*}
where the first equality uses $f(\bX_\star)=0$ and the second inequality follows from the rank-$2r$ mixed-norm RIP; thus $\mu = \delta_1$.

\section{Proof of Proposition~\ref{prop:mixedRIP_corrupted}}\label{proof:prop_mixedRIP_corrupted}

For $\bX_1$ and $\bX_2$ with $\rank(\bX_1 - \bX_2) \le 2r$, we have
\begin{align*}
|f(\bX_{1})-f(\bX_{2})| &= \Big| \| \cA(\bX_{1}-\bX_{\star}) -\bw -\bs \|_1 -\| \cA(\bX_{2}-\bX_{\star})  -\bw -\bs \|_1 \Big| \\
& \le \|\cA(\bX_{1}-\bX_{2})\|_1 \le \delta_{2}\|\bX_{1}-\bX_{2}\|_{\fro},
\end{align*}
where the second line follows from the inverse triangle inequality and the rank-$2r$ mixed-norm RIP; hence $L=\delta_2$. 
For approximate restricted sharpness, note that
\begin{align*}
f(\bX)-f(\bX_{\star}) & = \|\cA(\bX-\bX_{\star})-\bw - \bs \|_{1} - \|\bw +\bs \|_{1} \\ 
 & \ge \|\cA(\bX-\bX_{\star}) - \bs \|_{1} -\|\bw\|_1 - \| \bs \|_{1} - \|\bw\|_1 \\ 
 & = \|\cA_{\cS^{c}}(\bX-\bX_{\star})\|_{1}+\|\cA_{\cS}(\bX-\bX_{\star})-\bs \|_{1} - \|\bs \|_{1} -2\|\bw\|_1 \\
 & \ge \|\cA_{\cS^{c}}(\bX-\bX_{\star})\|_{1} - \|\cA_{\cS}(\bX-\bX_{\star})\|_{1}-2\|\bw\|_1 \\
 & \ge \delta_3 \|\bX-\bX_{\star}\|_{\fro} -2\|\bw\|_1\\
 & \ge \delta_3 \|\bX-\bX_{\star}\|_{\fro} -2\sigma_w,
\end{align*}
where the second and the fourth lines follow from the triangle inequality, the third line follows from the definition of $\cS$, and the last line follows from the definition of the $\cS$-outlier bound and the noise upper bound $\|\bw\|_1 \le \sigma_w$. Therefore, we have $\mu=\delta_3$ and $\xi = 2\sigma_w$.

\end{document}